\documentclass{article}

\usepackage{arxiv}

\usepackage[utf8]{inputenc} % allow utf-8 input
\usepackage[T1]{fontenc}    % use 8-bit T1 fonts
\usepackage{subfiles}
\usepackage{listings}
\usepackage{microtype}
\usepackage{graphicx}
\usepackage{booktabs} % for professional tables
\usepackage{amsmath}
\usepackage{graphicx}
\usepackage{enumerate}
\usepackage{pgfplots}
\usepackage{mathtools}
\usepackage{lipsum}
\usepackage{amssymb}
\usepackage{float}
\usepackage{multicol}
\usepackage{xspace}
\usepackage{lettrine}
\usepackage{tcolorbox}
\usepackage{subfigure}
\usepackage{tikz}
\usepackage{todonotes}
\usetikzlibrary{backgrounds,automata}
\usepackage{latexsym}
\usetikzlibrary{decorations.markings}
\usepackage{fancybox}
\usepackage[ruled,vlined]{algorithm2e}
\usepackage{amsthm}
\usepackage{hyperref}
\usepackage{algorithmic}
\usepackage{diagbox}
\usepackage{tabularx}
\usepackage{epsfig}
\usepackage[outdir=./]{epstopdf}

\newtheorem{theorem}{Theorem}
\newtheorem{proposition}{Proposition}
\newtheorem{corollary}{Corollary}

\newtheorem{remark}{Remark}

\newcommand{\var}{\rm Var}

\DeclareMathOperator*{\esssup}{ess\,sup}
\DeclareMathOperator\supp{supp}
\title{Cumulant GAN}

\author{
Yannis Pantazis \\
Institute of Applied and Computational Mathematics\\
Foundation for Research \& Technology - Hellas, Greece \\
  \texttt{pantazis@iacm.forth.gr} \\
   \And
 Dipjyoti Paul \\
 Computer Science Department\\
 University of Crete, Greece\\
  \texttt{dipjyotipaul@csd.uoc.gr} \\
   \And
 Michail Fasoulakis \\
 Institute of Computer Science\\
 Foundation for Research \& Technology - Hellas, Greece \\
  \texttt{mfasoul@ics.forth.gr} \\
    \And
 Yannis Stylianou \\
 Computer Science Department\\
 University of Crete, Greece \\
  \texttt{yannis@csd.uoc.gr} \\
   \And
 Markos Katsoulakis \\
 Department of Mathematics and Statistics\\
 University of Massachusetts at Amherst, USA \\
  \texttt{markos@math.umass.edu} \\
}
\date{}
\begin{document}
\maketitle

\begin{abstract}
In this paper, we propose a novel loss function for training Generative Adversarial Networks (GANs) aiming towards deeper theoretical understanding as well as improved stability and performance for the underlying optimization problem. The new loss function is based on cumulant generating functions giving rise to \emph{Cumulant GAN}. Relying on a recently-derived variational formula, we show that the corresponding  optimization problem is equivalent to R{\'e}nyi divergence minimization, thus offering a (partially) unified perspective of GAN losses: the R{\'e}nyi family encompasses Kullback-Leibler divergence (KLD), reverse KLD, Hellinger distance and $\chi^2$-divergence. Wasserstein GAN is also a member of cumulant GAN. In terms of stability, we rigorously prove the linear convergence of cumulant GAN to the Nash equilibrium for a linear discriminator, Gaussian distributions and the standard gradient descent ascent algorithm. Finally, we experimentally demonstrate that image generation is more robust relative to Wasserstein GAN and it is substantially improved in terms of both inception score and Fr\'echet inception distance  when both weaker and stronger discriminators are considered.
\end{abstract}


\section{Introduction}
A GAN is a two-player zero-sum game between a \emph{discriminator} and a \emph{generator}, both being neural networks with high learning capacity. GANs \cite{GPMXWOCB14} are powerful generative models capable of drawing new samples from an unknown distribution when only samples from that distribution are available. Their popularity stems from their ability to generate realistic samples from high-dimensional and complex distributions. In computer vision, GANs have been applied for (conditional) image generation \cite{mirza2014conditional,Denton2015LapGAN,radford2015unsupervised,odena2016conditional,karras2017progressive,brock2018large,karras2019style}, image synthesis from text (i.e., reverse captioning) \cite{pmlr-v48-reed16}, image-to-image translation \cite{cyclegan} and image super-resolution \cite{Ledig2016a}. In time-series data, GANs have been used for speech enhancement \cite{pascual2017segan}, speech synthesis \cite{saito2018statistical,kumar2019melgan} as well as for natural language processing \cite{che2017maximum, fedus2018maskgan} among other types of raw data. Several surveys and reviews on GANs are available in the literature \cite{Creswell_2018, 10.1145/3439723, journals/corr/abs-2001-06937, Jabbar2020ASO}. 
Moreover, the concept of adversarial training which essentially designates that the loss function is learned and thus data-driven and not predetermined by the user has been successfully applied in domain adaptation \cite{cyclegan,JMLR:v17:15-239,tzeng2017adversarial,pmlr-v80-hoffman18a} and representation disentanglement \cite{infogan}.

There are three ingredients that constitute a GAN: the architectures for both the generator and the discriminator, the training algorithm and the loss function which is further divided into the \emph{objective functional} to be optimized and the \emph{function space} where the discriminator belongs to. Over the years, the capacity of the neural networks has been increased resulting in significant gains in terms of naturalness and performance \cite{karras2017progressive,brock2018large,karras2019style,pmlr-v97-zhang19d}. Similarly, new normalization techniques such as spectral normalization \cite{miyato2018spectral} and new optimization algorithms \cite{daskalakis2018training} have been proposed. Two characteristic examples are progressively-learning GAN \cite{karras2017progressive} where the models are built in progressive levels of resolution and MelGAN \cite{kumar2019melgan} where weight normalization played a critical role for the generation of high quality speech. Several heuristics have been also devised \cite{salimans2016improved} to alleviate the difficulties of training GANs.

As already stated, the third ingredient in GANs' definition corresponds to the loss function. Since  their introduction, GANs have been described as a tractable approach to minimize a divergence or a distance between the real data distribution and the model distribution. Indeed, the original formulation of GANs \cite{GPMXWOCB14} can be seen as the minimization of the \emph{Shannon-Jensen divergence}. $f$-GAN \cite{nowozin2016f} is a generalization of vanilla GAN where a variational lower bound for the $f$-divergence is minimized. \emph{Least-Squares GAN} \cite{mao2017least} which minimizes a softened version of the Pearson $\chi^2$-divergence and hinge loss \cite{corr/LimY17} propose objective functionals aiming towards avoiding mode collapse issues. As it is well-documented, the training procedure of GANs often fails and several studies have suggested remedies to alleviate the observed hindrances. For instance, a recurring impediment with GAN training is the oscillatory behavior of the optimization algorithms due to the fact that the optimal solution is a saddle point of the loss function. Standard optimization algorithms such as stochastic gradient descent ascent (SGDA) may fail even for simple loss functions \cite{mertikopoulos2018cycles,daskalakis2018training}. 

One of the most successful approaches to improve stability is through the restriction of the function space of the discriminator. \emph{Wasserstein GAN} (WGAN) \cite{arjovsky2017wasserstein} which has been further improved in \cite{gulrajani2017improved} aims to minimize the Wasserstein distance which is equivalent to restricting to Lipschitz continuous functions. SNGAN \cite{miyato2018spectral} also restricts to Lipschitz continuous functions while MMD-GAN \cite{NIPS2017_dfd7468a} restricts the discriminator to belong to a reproducible Hilbert kernel space. Recently, the dissociation between the objective functional and the function space has been presented in a rigorous mathematical framework \cite{birrell2021fgammadivergences}. In this paper, we concentrate on the loss function and propose a new objective functional that further improves training stability and avoids mode collapse.

The novel objective functional is based on cumulant generating functions with the resulting model  referred as \emph{Cumulant GAN}.
A key advantage of cumulants over expectations is that cumulants capture \emph{higher-order} information about the underlying distributions which often results in more effective learning. Using this property, we rigorously prove that cumulant GAN converges exponentially fast when the gradient descent ascent algorithm is used for the special case with linear generator, linear discriminator and Gaussian distributions. Despite being a simple case, this theoretical result offers a rigorous and valuable differentiation between WGAN, which fails to converge, and the proposed cumulant GAN that demonstrates linear convergence to the  Nash equilibrium,  when the same gradient descent ascent algorithm is used on both.

Interestingly, the optimization of cumulant GAN  can be described as a \emph{weighting} extension of the standard stochastic gradient descent ascent where the samples that confuse the discriminator the most receive a higher weight, thus, contributing more to the update of the neural network's parameters. Furthermore, by applying a recent variational representation formula \cite{birrell2019distributional}, we show that cumulant GAN is capable of interpolating between several GAN formulations, thus, offering a partially-unified mathematical framework. Indeed, the optimization of the proposed loss function is equivalent to the minimization of a divergence for a wide set of cumulant GAN's hyper-parameter values. 
It is also worth-noting that despite  $f$-GAN's (partial) unification property \cite{nowozin2016f}, cumulant GAN and $f$-GAN formulations are not equivalent even when they minimize the same divergence and there is a subtle but important difference: the underlying variational representation which is eventually optimized is different. Ours is based on the Donsker-Varadhan representation formula while $f$-GAN is based on the Legendre transform of $f$-divergence. For KLD, Donsker-Varadhan formula is tighter than Legendre duality formula\footnote{Simply by the fact that $x\ge e\log x$; see also \cite{belghazi2018mutual}.}. Additionally, our formulation is computationally more manageable because the hyper-parameters of cumulant GAN are of continuous nature while $f$-GAN requires different $f$'s for different divergences.

Our numerical demonstrations aim to provide insights into cumulant GAN's representational ability and learnability advantages. Experiments on synthetic multi-modal data revealed the differences in the dynamics of learning for different hyper-parameter values of cumulant GAN. Even though the optimal solution is the same, the SGDA's dynamics for the training parameters driven by the chosen hyper-parameters' values resulted in very different distributional realizations with the two extremes being mode covering and mode selection. Moreover, using cumulant GAN, we were able to recover higher-order statistics even when the discriminator is linear. Finally, we demonstrated increased robustness as well as improved performance on image generation for both CIFAR10 and ImageNet datasets. We performed relative comparisons with WGAN under standard as well as distressed settings which is a primary  reason for training instabilities in GANs and demonstrated that not only cumulant GAN  is more stable but also it is better up to 68\% in terms of averaged inception score and up to 75\% in terms of Fr\'echet inception distance.

The paper is organized as follows. Section \ref{background:sec} introduces the necessary background theory, while Section \ref{cumgan:sec} defines cumulant GAN and highlights the derivation of several of its theoretical properties. In Section \ref{demo:sec}, numerical simulations on both synthetic and real datasets are presented, while Section \ref{concl:sec} concludes the paper.

\section{Background}
\label{background:sec}
The proposed GAN is a substantial generalization of  WGAN by means  of cumulant generating functions. These concepts are briefly discussed in this section.

\subsection{Wasserstein GAN}
WGAN \cite{arjovsky2017wasserstein,gulrajani2017improved} minimizes the Earth-Mover (a.k.a. 1-Wasserstein) distance and primarily aims to avoid gradient saturation during the training process. Based on the Kantorovich-Rubinstein duality formula for Wasserstein distance, the loss function of WGAN can be written as
\begin{equation}
\min_{G} \max_{D\in\mathcal{D}} \mathbb E_{p_{r}} [D(x)] - \mathbb E_{p_{g}} [D(x)],
\label{wgan:loss:eq}
\end{equation}
where $p_r$ and $p_g$ correspond to the real data distribution and the implicitly-defined model distribution, respectively. Namely, $p_g$ denotes the distribution of $G(z)$, where $G$ is the generator and $z\sim p_z(z)$ is a random input vector often following a standard normal or uniform distribution.
$D(\cdot)$ is the discriminator (called critic in the WGAN setup) while $\mathcal{D}$ is the function space of all 1-Lipschitz continuous functions. In WGAN, Lipschitz continuity is imposed by adding a (soft) regularization term on gradient values called Gradient Penalty (GP). It has been shown that GP regularization produces superior performance relative to weight clipping \cite{gulrajani2017improved}.

\subsection{Cumulant Generating Functions}
The cumulant generating function (CGF), also known as the log-moment generating function, is defined for a random variable with probability density function $p(x)$ as
\begin{equation}
\Lambda_{f,p} (\beta) = \log \mathbb E_p[e^{\beta f(x)}] \ ,
\label{cgf:eq}
\end{equation}
where $f$ is a measurable function with respect to $p$. The standard CGF is obtained when $f(x)=x$. CGF is a convex function with respect to  $\beta$ and it contains information for all moments of $p$. CGF also encodes   the tail behavior  of distributions and   plays a key role in the theory of Large Deviations for  the estimation of %extremely improbable,
rare events \cite{dupuis2011weak}. A power series expansion of the CGF reveals that the lower order statistics dominate when $|\beta|\ll1$ while all statistics contribute to the CGF when $|\beta|\gg 1$.
In statistical mechanics, CGF is the logarithm of the partition function, $-\beta^{-1} \Lambda_{f,p} (-\beta)$ is called the Helmholtz free energy where $\beta$ is interpreted as the inverse temperature and $f$ as the Hamiltonian \cite{stoltz2010free}. 
Furthermore, it is straightforward to show that $\Lambda_{f,p} (0)=0$ as well as $\Lambda_{f,p}' (0)=\mathbb E_p[ f(x)]$, hence, the following limit for CGF holds
\begin{equation}
\lim_{\beta\to0} \beta^{-1} \Lambda_{f,p} (\beta) =  \mathbb E_p[ f(x)] \ .
\label{cgf:lim:eq}
\end{equation}
We are now ready to introduce the new  GAN.

\section{Cumulant GAN}
\label{cumgan:sec}

\subsection{Definition}
We define a novel GAN model by substituting the expectations in the loss function of WGAN with the respective CGFs. Thus, we propose to optimize the following mini-max problem:
\begin{equation}
\begin{aligned}
&\min_{G} \max_{D\in\mathcal D} \ \left\{(-\beta)^{-1} \Lambda_{D,p_r} (-\beta) - \gamma^{-1} \Lambda_{D,p_g} (\gamma) \right\} \equiv\\
&\min_{G} \max_{D\in\mathcal D} \ \underbrace{-\beta^{-1} \log \mathbb E_{p_{r}} [e^{-\beta D(x)}] - \gamma^{-1} \log \mathbb E_{p_{g}} [e^{\gamma D(x)}]}_{=L(\beta,\gamma)},
\end{aligned}
\label{cumgan:loss:eq}
\end{equation}
%\noindent
where the hyper-parameters $\beta$ and $\gamma$ are two non-zero real numbers which control the learning dynamics as well as the optimal solution. Since the loss function is the difference of two CGFs, we call $L(\beta,\gamma)$ in (\ref{cumgan:loss:eq}) the \emph{cumulant loss function} and the respective generative model as \emph{Cumulant GAN}. Throughout this paper, we assume the mild condition that both CGFs are finite for a neighborhood of $(0,0)$, therefore, the cumulant loss is well-defined for $|\beta|+|\gamma|<\epsilon$, for some $\epsilon>0$.

The definition of the loss function is extended  on the axes  and the origin of the $(\beta,\gamma)$-plane using the limit in  (\ref{cgf:lim:eq}). Hence, the cumulant loss function is defined for all  values of $\beta$ and $\gamma$ for which the new loss function is finite. It is straightforward to show that WGAN is a special case of cumulant GAN.

\begin{proposition}
Let $\mathcal D$ be the set of all 1-Lipschitz continuous functions. Then, cumulant GAN with $(\beta,\gamma)=(0,0)$ is equivalent to WGAN.
\end{proposition}
\begin{proof}
The proposition is a consequence of the fact that
\begin{equation*}
\lim_{\beta,\gamma\to 0} L(\beta,\gamma) = L(0,0) = \mathbb E_{p_{r}} [D(x)] - \mathbb E_{p_{g}} [D(x)].
\end{equation*}
\end{proof}

\begin{remark}
The same proof applies when $\mathcal D$  is the set of all measurable and bounded functions and cumulant GAN with $(\beta,\gamma)=(0,0)$ is equivalent to minimizing the Radon metric between the two distributions which corresponds to the origin in Fig. \ref{beta:gamma:plane:fig}.
\end{remark}

Next, we rigorously demonstrate that cumulant GAN can be seen as a unified and smooth interpolation between several well-known  divergence minimization problems.

\subsection{KLD, Reverse KLD and R{\'e}nyi Divergence as Special Cases}
A major inconvenience of many GAN formulations is their inability to interpret the loss function value and understand the properties of the obtained solution. Even when the stated goal is to minimize a divergence as in the  original GAN and the $f$-GAN, the utilization of training tricks such as a non-saturating generators may result in the minimization of something completely different as it was recently observed \cite{shannon2020the}. In contrast, the proposed cumulant loss function can be interpreted for several choices of its hyper-parameters. Below there is a list of values for $\beta$ and $\gamma$ that result to interpretable loss functions. Indeed, several well-known divergences are recovered when the function space for the discriminator is the set of all measurable and bounded functions. In the following, we make the convention that a forward divergence or simply divergence is a divergence that uses the probability ratio, $\frac{p_r}{p_g}$, while a reverse divergence uses the reciprocal ratio.

\begin{theorem}
Let $\mathcal D$ be the set of all bounded and measurable functions. Then, the optimization of cumulant loss in (\ref{cumgan:loss:eq}) is equivalent to the minimization of
\begin{enumerate}
\item[a.] \underline{Kullback-Leibler divergence} for $(\beta,\gamma)=(0,1)$:
\[
\min_{G} \max_{D\in\mathcal D} L(0,1) \equiv \min_{G} D_{KL} \left(p_r||p_g\right).
\]

\item[b.] \underline{Reverse KLD} for $(\beta,\gamma)=(1,0)$:
\[
\min_{G} \max_{D\in\mathcal D} L(1,0) \equiv \min_{G} D_{KL} \left(p_g||p_r\right).
\]

\item[c.] \underline{R{\'e}nyi divergence} for $(\beta,\gamma)=(\alpha,1-\alpha)$ with $\alpha\neq 0$ and $\alpha\neq 1$:
\[
\min_{G} \max_{D\in\mathcal D} L(\alpha,1-\alpha) \equiv \min_{G} \mathcal R_{\alpha} \left(p_g||p_r\right),
\]
as well as for $(\beta,\gamma)=(1-\alpha,\alpha)$ with $\alpha\neq 0$ and $\alpha\neq 1$:
\[
\min_{G} \max_{D\in\mathcal D} L(1-\alpha,\alpha) \equiv \min_{G} \mathcal R_{\alpha} \left(p_r||p_g\right),
\]
where $\mathcal R_{\alpha} \left(p||q\right)$ is the R{\'e}nyi divergence defined by
\[
\mathcal R_{\alpha} \left(p||q\right) = \frac{1}{\alpha(1-\alpha)}\log \mathbb E_q\left[\left(\frac{p}{q}\right)^\alpha\right],
\]
when $p$ and $q$ are absolutely continuous with respect to each other and $\alpha>0$\footnote{The definition is extended for $\alpha<0$ using the symmetry identity
$\mathcal R_{\alpha} \left(p||q\right) = \mathcal R_{1-\alpha} \left(q||p\right)$.}.
\end{enumerate}
\end{theorem}

\begin{proof}
a. Using the definition of $L(\beta, \gamma)$, we have: 
%It holds that
\begin{equation*}
\begin{aligned}
\max_{D\in\mathcal D} L(0,1)
&= \max_{D\in\mathcal D} \left\{ \mathbb E_{p_r}[D(x)] -\log \mathbb E_{p_g} [e^{D(x)}]\right\} \\
&= D_{KL} \left(p_r||p_g\right),
\end{aligned}
\end{equation*}\noindent
where the last equation is the Donsker-Varadhan variational formula \cite{donsker1983asymptotic,dupuis2011weak}.

\noindent
b. Similarly,

\begin{equation*}
\begin{aligned}
\max_{D\in\mathcal D} L(1,0)
&= \max_{D\in\mathcal D} \left\{-\log \mathbb E_{p_r} [e^{-D(x)}] - \mathbb E_{p_g}[D(x)]\right\} \\
&= \max_{D'=-D\in\mathcal D} \left\{ \mathbb E_{p_g}[D'(x)] -\log \mathbb E_{p_r} [e^{D'(x)}]\right\} \\
&= D_{KL} \left(p_g||p_r\right),
\end{aligned}
\end{equation*}
where we applied again the Donsker-Varadhan variational formula.

\noindent
c. Generalizing a. and b. we now have:
{\small
\begin{equation*}
\begin{aligned}
&\max_{D\in\mathcal D} L(\alpha,1-\alpha) \\
=& \max_{D\in\mathcal D} \left\{-\frac{1}{\alpha}\log \mathbb E_{p_r} [e^{-\alpha D(x)}] - \frac{1}{1-\alpha}\log \mathbb E_{p_g} [e^{(1-\alpha) D(x)}] \right\} \\
=& \max_{D'=-D\in\mathcal D} \left\{ \frac{1}{\alpha-1}\log \mathbb E_{p_g} [e^{(\alpha-1) D'(x)}] - \frac{1}{\alpha}\log \mathbb E_{p_r} [e^{\alpha D'(x)}]\right\} \\
=& \ \mathcal R_{\alpha} \left(p_g||p_r\right),
\end{aligned}
\end{equation*}}\noindent
where the last equation is an extension of the  Donsker-Varadhan variational formula to   R{\'e}nyi divergence and was recently proved in (\cite[Theorem 3.1]{birrell2019distributional}). For completeness, we provide a  proof of the R{\'e}nyi divergence variational representation in Appendix A of Supplementary Materials.

The proof  for the case $L(1-\alpha,\alpha)$ is similar and agrees with  the symmetry identity for the R{\'e}nyi divergence,  $\mathcal R_{\alpha} \left(p||q\right) = \mathcal R_{1-\alpha} \left(q||p\right)$.
\end{proof}

%We can slightly generalize a. and b.and extend both KLD and reverse KLD to $L(0,\gamma)$ and $L(\beta,0)$, respectively, with $\beta,\gamma>\epsilon>0$.

%\tdi{Can also modify $\mathcal{D}$ to make connections to new concept of $\Gamma$-divergence.}

The R{\'e}nyi divergence, $\mathcal R_{\alpha}$, interpolates between KLD ($\alpha\to0$) and reverse KLD ($\alpha\to1$). Interestingly, there are additional special cases that belong to the family of R{\'e}nyi divergences. The following corollary states some of them, while Fig.~\ref{beta:gamma:plane:fig} depicts schematically the obtained divergences and distances on the $(\beta,\gamma)$-plane.

\begin{corollary}
Under the same assumption as in Theorem 1, the optimization of (\ref{cumgan:loss:eq}) is equivalent to the minimization of
\begin{enumerate}
\item[a.] \underline{\it Hellinger distance} for $(\beta,\gamma)=(\frac{1}{2},\frac{1}{2})$:
{\[
\min_{G} \max_{D\in\mathcal D} L(\frac{1}{2},\frac{1}{2}) \equiv \min_{G} -4\log\left(1-D_{H}^2 \left(p_g,p_r\right)\right),
\]}
where {$D_{H}^2 \left(p,q\right)= \frac{1}{2}\mathbb E_{q}\left[\left((\frac{p}{q})^{1/2}-1\right)^2\right]$} is the square of the Hellinger distance \cite{tsybakov2008introduction}.

\item[b.] \underline{$\chi^2$-divergence} for $(\beta,\gamma)=(-1,2)$:
{\[
\min_{G} \max_{D\in\mathcal D} L(-1,2) \equiv \min_{G} \frac{1}{2}\log \left(1 + \chi^2 \left(p_r||p_g\right)\right),
\]}
and \underline{reverse $\chi^2$-divergence} for $(\beta,\gamma)=(2,-1)$:
{\[
\min_{G} \max_{D\in\mathcal D} L(2,-1) \equiv \min_{G} \frac{1}{2}\log \left(1 + \chi^2 \left(p_g||p_r\right)\right),
\]} 
where {$\chi^2 \left(p||q\right)=\mathbb E_{q}\left[(\frac{p}{q}-1)^2\right]$} is the $\chi^2$-divergence\footnote{Forward $\chi^2$-divergence is often called Pearson $\chi^2$-divergence while the reverse $\chi^2$-divergence is often called Neyman $\chi^2$-divergence.} \cite{tsybakov2008introduction}.

\item[c.] \underline{All-mode covering} or worst-case regret in minimum description length principle \cite{grunwald2007minimum} for $(\beta,\gamma)=(\infty,-\infty)$:
{\small\[
\min_{G} \lim_{\alpha\to\infty}\alpha \max_{D\in\mathcal D} L(\alpha,1-\alpha) \equiv \min_{G} \log \left(\esssup_{x\in\supp(p_g)} \frac{p_g(x)}{p_r(x)}\right),
\]}
where $\esssup$ is the essential supremum of a function.

\item[d.] \underline{Largest-mode selector} for $(\beta,\gamma)=(-\infty,\infty)$:
{\small\[
\min_{G} \lim_{\alpha\to\infty} \alpha \max_{D\in\mathcal D} L(1-\alpha,\alpha) \equiv \min_{G} \log \left(\esssup_{x\in\supp(p_r)} \frac{p_r(x)}{p_g(x)}\right).
\]}

\end{enumerate}
\end{corollary}

\begin{proof} 
All cases a.-d. follow from Theorem 1.c as special instances of R{\'e}nyi divergence:\\ $R_{1/2}(p\|q)=-4\log\left(1-D_{H}^2 \left(p,q\right)\right)$, \\
$R_{2}(p\|q)=\frac{1}{2}\log \left(1 + \chi^2 \left(p||q\right)\right)$, $R_{-1}(p\|q)=R_{2}(q\|p)$\\
and $\lim_{\alpha \to \infty} \alpha R_{\alpha}(p\|q)=\log \left(\esssup_{x\in\supp(q)} \frac{p(x)}{q(x)}\right)$. \\
%and \\
%$\lim_{\alpha \to -\infty} \alpha R_{\alpha}(p\|q)=\log
%\left(\esssup_{x\in\supp(p)} \frac{q(x)}{p(x)}\right)$. \\
We refer to \cite{minka2005divergence,bishop2006pattern} and the references therein for detailed proofs.
\end{proof}

The flexibility of the two hyper-parameters is significant since  they offer a simple recipe to remedy some of the most frequent issues of GAN training. For instance, KLD tends to cover all the modes of the real distribution while reverse KLD tends to select a subset of them \cite{minka2005divergence,bishop2006pattern,hernandez2016black,li2016renyi,shannon2020the} (see also Fig.~\ref{kld:vs:rkld:fig} for a benchmark). Therefore, if mode collapse is observed during training, then, increasing $\gamma$ with $\beta=1-\gamma$ will push the generator towards generating a wider variety of samples. 
In the other limit, more realistic samples (e.g. less blurry images)  but with  less variability will be generated when $\beta$ is increased while $\gamma=1-\beta$.

We also expand the interpretation of the $(\beta,\gamma)$-plane to the case where $\beta+\gamma>0$ as the following proposition demonstrates. We show that the half-lines beginning at the origin and passing through the line $\beta+\gamma=1$ define the same divergence therefore we call them divergence rays or {\bf d-rays} for shorthand as depicted in Fig. \ref{beta:gamma:plane:fig} (gray half-lines).

\begin{proposition}
Let $\alpha\in\mathbb R\smallsetminus\{0,1\}$ and $\delta\in[\delta_{\min},\delta_{\max}]$ with $0<\delta_{\min}<\delta_{\max}<\infty$ and $\mathcal D$ be the set of all bounded and measurable functions. Then, the optimization of cumulant loss in (\ref{cumgan:loss:eq}) is equivalent to the minimization of {\it scaled R\'enyi divergence} for $(\beta,\gamma)=(\delta(1-\alpha), \delta\alpha)$:
\[
\min_{G} \max_{D\in\mathcal D} L(\delta(1-\alpha), \delta\alpha) \equiv \min_{G}  \frac{1}{\delta} \mathcal R_{\alpha} \left(p_r||p_g\right),
\]
\end{proposition}
\begin{proof}
The maximization part of cumulant GAN becomes 
{\small
\begin{equation*}
\begin{aligned}
&\max_{D\in\mathcal D} L(\delta(1-\alpha), \delta\alpha) \\
=& \max_{D\in\mathcal D} \left\{-\frac{1}{\delta(1-\alpha)}\log \mathbb E_{p_r} [e^{-\delta(1-\alpha) D(x)}] - \frac{1}{\delta\alpha}\log \mathbb E_{p_g} [e^{\delta\alpha D(x)}] \right\} \\
=& \frac{1}{\delta} \max_{D'=\delta D\in\mathcal D} \left\{ \frac{1}{(\alpha-1)}\log \mathbb E_{p_r} [e^{(\alpha-1) D'(x)}] - \frac{1}{\alpha}\log \mathbb E_{p_g} [e^{\alpha D'(x)}]\right\} \\
=& \ \frac{1}{\delta} \mathcal R_{\alpha} \left(p_r||p_g\right),
\end{aligned}
\end{equation*}}\noindent
which completes the proof since $\delta$ is positive and far from 0 or $\infty$ thus $\delta D\in\mathcal D$.
\end{proof}

\begin{remark}
From a practical perspective, the boundedness condition required in the above theoretical formulation can be easily enforced by considering a clipped discriminator with clipping factor $M$, i.e., $D_M(x) = M\tanh(\frac{D(x)}{M})$. On the other hand, the set of all measurable functions is a very large class of functions and it might be difficult to be represented by a neural network. However, one can approximate measurable functions with continuous functions via Lusin's theorem \cite{Folland:book} which states that every finite Lebesgue measurable function is approximated arbitrarily well by a continuous function except on a set of arbitrarily small Lebesgue measure. Therefore, a sufficiently-large neural network can accurately approximate any measurable function.
\end{remark}

\begin{figure}[ht]
\begin{center}
\centerline{\includegraphics[width=.7\columnwidth]{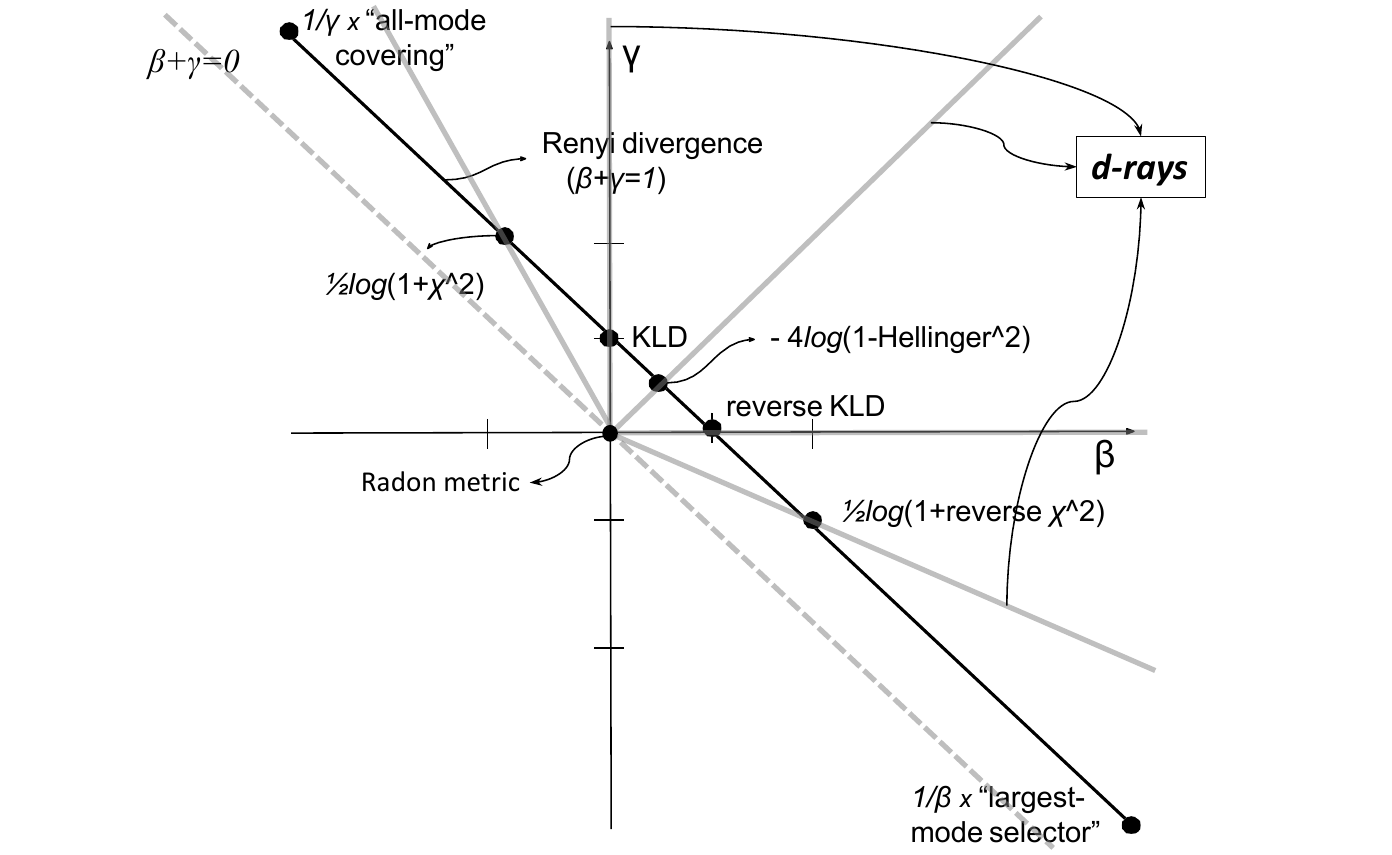}}
\caption{Special cases of \emph{cumulant GAN}. Line defined by $\beta+\gamma=1$ has a point symmetry. The central point, $(\frac{1}{2},\frac{1}{2})$, corresponds to the Hellinger distance. For each point, $(\alpha, 1-\alpha)$, there is a symmetric one, i.e., $(1-\alpha,\alpha)$, which has the same distance from the symmetry point. The respective divergences have reciprocal probability ratios (e.g., KLD \& reverse KLD, $\chi^2$-divergence \& reverse $\chi^2$-divergence, etc.). Each point on the ray starting at the origin and pass through the point $(\alpha, 1-\alpha)$ also corresponds to (scaled) R\'enyi divergence of order $\alpha$. These half-lines are called d-rays.}
\label{beta:gamma:plane:fig}
\end{center}
\end{figure}

\subsection{Cumulant GAN as a Weighted Version of the SGDA Algorithm}

The parameter estimation for the cumulant GAN is performed using the SGDA algorithm. Algorithm~\ref{cumgan:alg} presents the core part of SGDA's update steps where we exclude any regularization terms for clarity purposes. Namely, $\eta$ and $\theta$ are the parameters of the discriminator and the generator, respectively, while $\lambda$ is the learning rate. The proposed loss function is not the difference of two expected values, therefore, the order between differentiation and expectation approximation does matter. We choose to first approximate the expected values with the respective statistical averages as
\begin{equation}
\hat{L}_m(\beta,\gamma) = -\frac{1}{\beta} \log \sum_{i=1}^m e^{-\beta D(x_i)} - \frac{1}{\gamma} \log \sum_{i=1}^m e^{\gamma D(G(z_i))} \ . 
\label{cumgan:loss:stat:eq}
\end{equation}\noindent
Then, we apply the differentiation operator which results in a weighted version of SGDA as shown in Algorithm~\ref{cumgan:alg}. Interestingly, several recent papers \cite{burda2015importance,li2016renyi,hu2018on,hjelm2018boundary,pantazis2019training} included  a weighting perspective in their optimization approach.

\begin{algorithm}[th]
\caption{Core of SGDA Iteration}
\label{cumgan:alg}
\begin{algorithmic}
   \STATE {\bfseries Input:} data batch: $\{x_i\}$, noise batch: $\{z_i\}$
   \FOR{$k$ steps}
   \STATE 
   {\footnotesize\begin{equation*} 
   \eta \leftarrow \eta + \lambda \left(\sum_{i=1}^{m}w_i^\beta \nabla_{\eta} D\left(x_i\right) - \sum_{i=1}^{m}w_i^\gamma \nabla_{\eta} D\left(G\left(z_i\right)\right)\right) 
   \end{equation*}}
   \ENDFOR
   \STATE
   {\footnotesize\begin{equation*}
   \theta \leftarrow \theta + \lambda \left(\sum_{i=1}^{m}w_i^\gamma \nabla_{\theta} D\left(G\left(z_i\right)\right)\right) 
   \end{equation*}}
\end{algorithmic}
\end{algorithm}

The difference between WGAN and cumulant GAN for the update steps is the weights $w_i^\beta$ and $w_i^\gamma$. In WGAN, the weights are constant and equal to $\frac{1}{m}$ while in cumulant GAN they are defined for any $i=1,...,m$ by
{\[
w_i^\beta=\frac{e^{-\beta D\left(x_i\right)}}{\sum_{j=1}^m e^{-\beta D\left(x_j\right)}}, \ \ 
\text{and}, \ \ 
w_i^\gamma=\frac{e^{\gamma D\left(G(z_i)\right)}}{\sum_{j=1}^m e^{\gamma D\left(G(z_j)\right)}}.
\]}\noindent
The weights redistribute the sample distributions based on the assessment of the current discriminator. Fig.~\ref{cumgan:weights:fig} qualitatively demonstrates the change of the weight relative to uniform weights for $\beta,\gamma>0$. The weights place more emphasis  on  the real samples  associated with  the smallest $D(x_i)$ values. Similarly they place more emphasis on the synthetic samples that give the highest $D(G(z_i))$ values. A quantitative demonstration of the weights and how they evolve during the training process is presented in Appendix C (Figs. 1--3) of Supplementary Materials. 

The intuition behind the weighting mechanism is that samples that confuse the discriminator, i.e., the samples around the ``fuzzy'' decision boundary, are more valuable for the training process than  samples that are easily distinguished, thus, they should weigh more. Essentially, the discriminator is updated with samples produced by a better  generator than the current one, as well as with more challenging real samples. Similarly, the generator is also updated using samples from a generator which is better than the current one. 
Overall, due to the use of the weights $w_i^\beta, w_i^\gamma$ in Algorithm 1,  both generator and discriminator updates will be more affected  by synthetic samples that are more  indistinguishable from the real ones.

Additionally, the update of the discriminator is performed $k$ times more than the generator's update offering two important advantages. First, more iterations for the discriminator implies that it better distinguishes the real data from the generated ones, making the weighting perspective more valid. Second, it better approximates the optimal discriminator, thus, the theory presented in the previous section becomes more credible in practice. 

\begin{figure}[ht]
\begin{center}
\centerline{\includegraphics[width=.7\columnwidth]{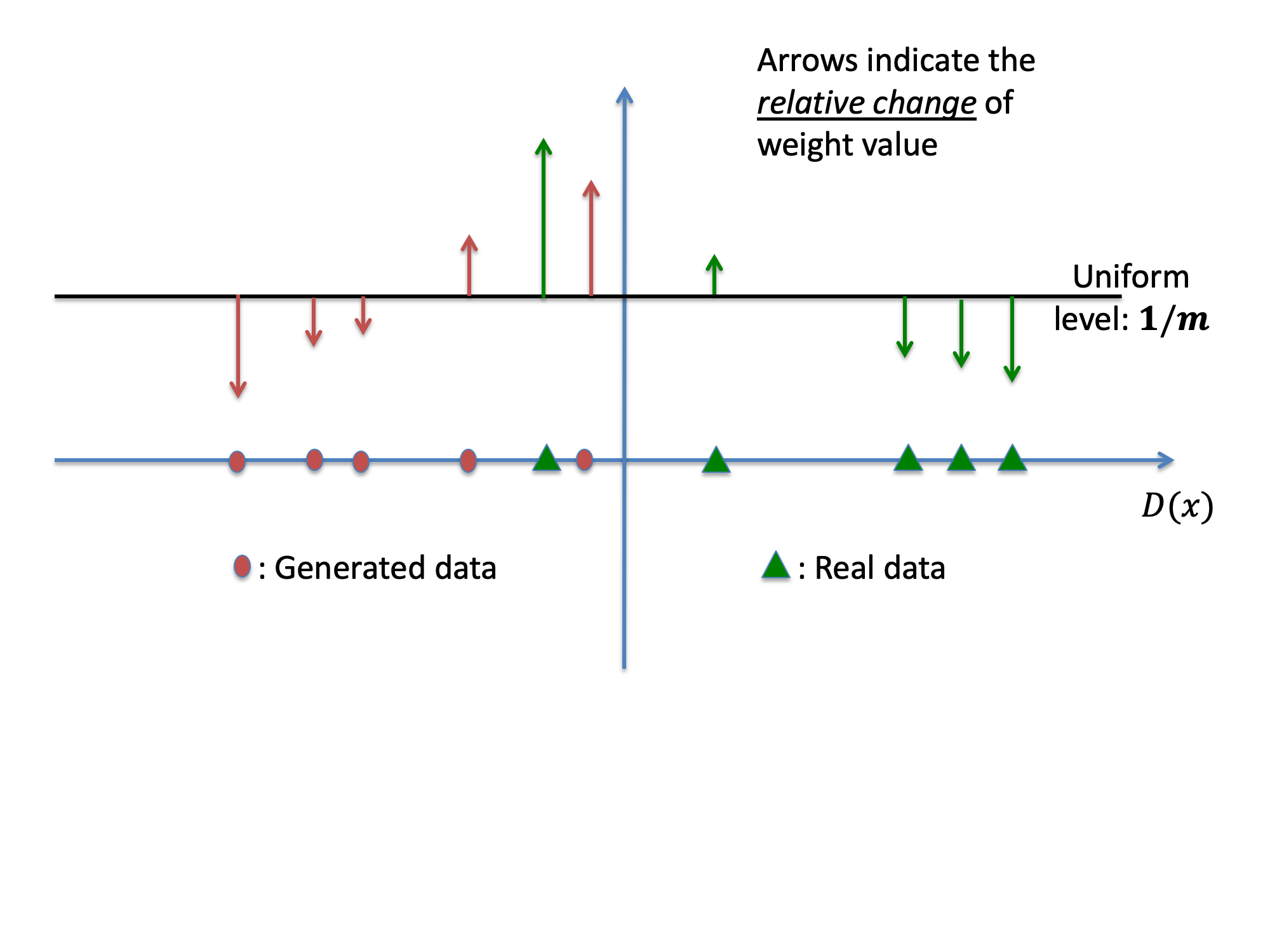}}
\vspace{-15mm}
\caption{Interpretation of \emph{cumulant GAN} as a weighted variation of SGDA for $\beta,\gamma>0$. Both real and generated samples for which the discriminator outputs a value closer to the decision boundary are assigned with larger weights because these are the samples which most probably confuse the discriminator.}
\label{cumgan:weights:fig}
\end{center}
\end{figure}

\begin{remark}
The Monte Carlo approximation in (\ref{cumgan:loss:stat:eq}) is biased. However, it has been shown that it is consistent \cite{li2016renyi}, hence, the error due to the statistical approximation decreases as the size of minibatch increases. Bias correction gradients using moving averages have been utilized in \cite{belghazi2018mutual} for the estimation of CGF. However, the modification of the loss function and the lack of an interpretation analogous to the weights  $w_i^\beta, w_i^\gamma$  are two key reasons to avoid inserting any bias-correction mechanism.
\end{remark}

\subsection{Convergence Guarantees for Linear Discriminator}
Let $\mathcal D$ be the set of all linear functions (i.e., $D(x)=\eta^T x$ with $\eta, x\in\mathbb R^d$) and assume that the real data follow a Gaussian distribution with mean value $\mu\in\mathbb R^d$ and covariance matrix, $I_d$. The generator is defined by $G(z) = z+\theta$, where $z$ is a standard $d$-dimensional Gaussian. The loss function for WGAN is\footnote{We did not add the gradient penalty in the current formulation since it has been shown that the convergence behavior of the gradient descent/ascend algorithm is not affected  \cite[Appendix B]{daskalakis2018training}.}
\begin{equation}
\min_\theta \max_\eta \ \eta^T(\mu-\theta), 
\label{wgan:loss:linear}
\end{equation}
while the respective exact cumulant loss function from (\ref{cumgan:loss:eq}) is given by \cite{Holmquist:1988}
\begin{equation}\small
\begin{aligned}
&\min_\theta \max_\eta \ \frac{1}{-\beta}(-\beta\eta^T\mu +\frac{1}{2}\beta^2\eta^T\eta) - \frac{1}{\gamma}(\gamma\eta^T\theta +\frac{1}{2}\gamma^2\eta^T\eta) \\
&\equiv \min_\theta \max_\eta \ \eta^T(\mu-\theta) -\frac{\beta+\gamma}{2}\eta^T\eta.
\label{cumgan:loss:linear}
\end{aligned}
\end{equation}
It has been proven that the training dynamics oscillates without converging to the optimum for the WGAN loss function (\ref{wgan:loss:linear}) if gradient descent ascent (GDA) 
%\footnote{The GDA algorithm is also known as primal-dual gradient descent.} 
is used \cite{mertikopoulos2018cycles} and more sophisticated algorithms such as training with optimism \cite{daskalakis2018training} or two-step extra-gradient approaches \cite{abs-1901-08511} are required to guarantee convergence. 
The use of CGFs transforms the optimization problem from just concave to a strongly concave problem for $\eta$. Actually, the cumulant loss function (\ref{cumgan:loss:linear}) is $\frac{\beta+\gamma}{2}$-strongly concave.
When the loss is both strongly convex and strongly concave, the GDA algorithm converges linearly (i.e., exponentially fast)  to the optimal solution under efficient proximal mappings admission \cite{siamjo/ChenR97}. 
Our case, where the loss is not strongly-convex with respect to $\theta$ but it is strongly-concave for $\eta$, has also linear convergence when the coupling term between $\eta$ and $\theta$ is full-column rank and the learning rates are properly chosen as it has been shown in \cite{DuHu:2019}. 
The following theorem demonstrates that the training dynamics for the cumulant loss function (\ref{cumgan:loss:linear}) converges even when the GDA algorithm uses the same learning rate for both players which is the main difference between  our result and  \cite{DuHu:2019}.
Next, without loss of generality we assume $\gamma=0$.
 
\begin{theorem}
The GDA method with learning rate $\lambda$ converges exponentially fast to the (unique) Nash equilibrium with rate  $1-(1-\epsilon)\lambda\beta$ if $\beta\in(\lambda/\epsilon, 1)$ with $\lambda<\epsilon<1$. Mathematically,  for the $t$-th iteration of GDA we have
\begin{equation}
||(\theta_t,\eta_t)-(\mu,0)||_2^2 \leq c (1-(1-\epsilon)\lambda\beta)^t,
\label{conv:rate:eq}
\end{equation}
where $(\theta^*,\eta^*)=(\mu,0)$ is the Nash equilibrium while $c$ is a computable positive constant.
\end{theorem}

\begin{proof}
The update step of GDA for the cumulant loss is given by
{\begin{equation*}
\begin{aligned}
\eta_{t+1} &= \eta_t + \lambda(\mu-\theta_t-\beta\eta_t) \ , \\
\theta_{t+1} &= \theta_t + \lambda\eta_t \ .
\end{aligned}
\end{equation*}}

\noindent
% Case $0<\beta\leq 1$: 
Define the energy function
{$$
E(\eta,\theta) = \eta^T\eta - \beta\eta^T(\mu-\theta) + (\mu-\theta)^T(\mu-\theta).
$$}\noindent
$E(\eta,\theta)$ is a second order polynomial for $\eta$; it is straightforward to show that if $0<\beta< 1$ then $E(\eta,\theta)\ge 0$ for all $\eta$ and $\theta$ and it is equal to 0 iff $\eta=\eta^*=0$ and $\theta=\theta^*=\mu$. Additionally, it generally holds that
{$$
||(\theta,\eta)-(\mu,0)||_2^2 \leq 2 E(\eta,\theta),
$$}\noindent
since $2 E(\eta,\theta)-||(\theta,\eta)-(\mu,0)||_2^2=\eta^T\eta - 2\beta\eta^T(\mu-\theta) + (\mu-\theta)^T(\mu-\theta)\geq 0$ for all $0<\beta< 1$.

Next, we show that $E(\eta_t,\theta_t)$ converges exponentially fast to $0$. Since, $E(\eta,\theta) = \sum_{i=1}^d \eta_i^2 - \beta\eta_i(\mu_i-\theta_i) + (\mu_i-\theta_i)^2$,  we can proceed with $d=1$ without sacrificing the generality of the proof. Using symbolic calculations, we obtain
\begin{equation*}
\begin{aligned}
E(\eta_{t+1},\theta_{t+1}) &= (1-(1-\epsilon)\lambda\beta) E(\eta_{t},\theta_{t}) \\
&- \lambda(\epsilon\beta-\lambda)[\eta_t^2 - \beta\eta_t(\mu-\theta_t) + (\mu-\theta_t)^2] \\
&\leq (1-(1-\epsilon)\lambda\beta) E(\eta_{t},\theta_{t}),
\end{aligned}
\end{equation*}
\noindent
since $\eta_t^2 - \beta\eta_t(\mu-\theta_t) + (\mu-\theta_t)^2 \geq 0$ for $\beta< 1$ and $\beta>\lambda/\epsilon$. The iterative application of this inequality yields
{\begin{equation*}
E(\eta_{t+1},\theta_{t+1}) \leq (1-(1-\epsilon)\lambda\beta)^{t+1} E(\eta_{0},\theta_{0}).
\end{equation*}}
\noindent
Combining the above inequalities we prove (\ref{conv:rate:eq}) with $c=2 E(\eta_{0},\theta_{0})$.
\end{proof}

Finally, we remark that similar to previous studies that prove linear convergence \cite{siamjo/ChenR97, DuHu:2019}, our proof utilizes the concept of energy functions (a.k.a. Lyapunov functionals), a tool from the theory of Dynamical Systems that has the potential to be transferable to more general optimization problems, too.

\begin{figure*}[ht]
\begin{center}
\centerline{\includegraphics[width=.95\columnwidth]{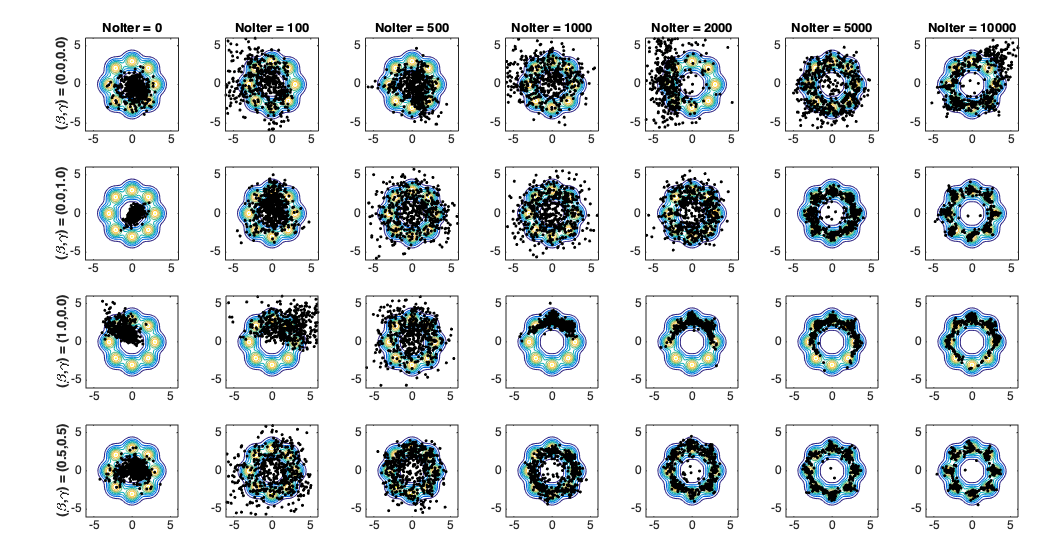}}
\caption{Generated samples using the Wasserstein distance using clipping (1st row), KL divergence (2nd row), reverse KLD (3rd row) and Hellinger distance (last row). The boundedness condition is not enforced on this example but it is necessary to be satisfied when the hyper-parameters take negative values.}
\label{kld:vs:rkld:fig}
\end{center}
\end{figure*}

\section{Demonstrations}
\label{demo:sec}

\subsection{Traversing the $(\beta,\gamma)$-plane: from  Mode Covering to Mode Selection}
As demonstrated in Section III.B and Fig.~\ref{beta:gamma:plane:fig}, the optimization of cumulant GAN for the set of bounded and measurable functions and various hyper-parameter values is equivalent to the minimization of a divergence. It is well-known that different divergences result in fundamentally different behavior of the solution. For instance, KLD minimization tends to produce a distribution that covers all the modes while the reverse KLD tends to produce a distribution that is focused on a subset of the modes \cite{minka2005divergence,bishop2006pattern,hernandez2016black}. Taking the extreme cases, an all-mode covering is obtained as $\beta\to-\infty$ while largest mode selection is observed at the other limit direction.

Our first example aims at highlighting the above characteristics of divergences and additionally to verify that the sub-optimal approximation of the function space of all bounded functions by a family of neural networks does not significantly affect the expected outcomes. Fig.~\ref{kld:vs:rkld:fig} presents generated samples for various values of the $(\beta,\gamma)$ pair at different stages of the training process as quantified by the number of iterations (denoted by `NoIter' in the Figure). The target distribution is a mixture of 8 equiprobable and equidistant-from-the-origin Gaussian random variables. Both discriminator and generator are neural networks with 2 hidden layers with 32 units each and ReLU as activation function. Input noise for the generator is an 8-dimensional standard Gaussian. In all cases, the discriminator is updated $k=5$ times followed by an update for the generator.

KLD minimization that corresponds to $(\beta,\gamma)=(0,1)$  (second row) tends to cover all modes while reverse KLD that corresponds to $(\beta,\gamma)=(1,0)$  (third row) tends to select a subset of them. This is particularly evident when the number of iterations is between 500 and 2000. Hellinger distance minimization (last row) produces samples with statistics that  lie between KLD and reverse KLD minimization while Wasserstein distance minimization (first row) has a less controlled behavior. It is also noteworthy that reverse KLD was not able to fully cover all the modes after 10K iterations. This behavior is not necessarily a drawback since the divergence of choice  is primarily an application-specific decision. For instance, the lack of diversity might be sacrificed in image generation for the sake of sharpness of the synthetic images.

We note that despite demonstrating a single run, the plots in Fig.~\ref{kld:vs:rkld:fig} are not cherry-picked. We have tested several architectures with more or fewer layers, as well as more or fewer  units per layer, repeating each run several times, with qualitatively similar results which are presented in Appendix E.A (Figs. 4--6) of Supplementary Materials. We further tested and compared the performance of various hyper-parameter values of cumulant GAN on two additional data distributions and presented them in Appendix E.B (Figs. 7--12). The first dataset is a mixture of 6 equiprobable Student's t distributions while the second dataset is the Swiss-roll distribution. Overall, cumulant GAN with $(\beta,\gamma)=(0.5,0.5)$ (i.e., Hellinger distance\footnote{Actually, we minimize $-4\log(1-Hel^2)$, see Corollary 1.}) generated the most accurate results for all datasets and across  various architectures. Finally, we experimented with d-rays and showed in Appendix E.C (Figs. 13--15) that the training process is qualitatively similar in terms of `mode covering' versus `mode selection' across the divergence rays.

\subsection{Learning the Covariance Matrix of a Multivariate Gaussian}
A CGF can uniquely determine a distribution and contains information on  all moments. Therefore, the use of simple discriminators which may fail under the WGAN loss might be sufficient under the cumulant loss in order to successfully train the generator. In this section, we provide an explicit example where the discriminator despite being a linear function the target is to learn the second order statistic of a multivariate Gaussian distribution. Thus, the real data, $x\in\mathbb R^d$, follow a zero-mean Gaussian with covariance matrix $\Sigma$, the discriminator is given by $D(x)=\eta^T x$ while the generator is given by $G(z) = Az$ where $A$ is a $d\times k$ matrix and $z$ is a standard $k$-dimensional Gaussian. The aim is to obtain a solution, $\hat{\Sigma} = \hat{A}\hat{A}^T$, close to the true covariance matrix.

The loss function of WGAN is $L(0,0) = \eta^T \mathbb E_{p_r}[x] - \eta^T A \mathbb E_{p_z}[z] = 0$, therefore it is impossible here to learn the covariance matrix. On the other hand, the cumulant loss reads
\begin{equation}
L(\beta,\gamma) = -\frac{1}{2}\eta^T (\beta\Sigma + \gamma AA^T) \eta
\end{equation}
\noindent
allowing the possibility of a $(\beta,\gamma)$ pair that makes the Nash equilibrium non-trivially informative regarding the covariance matrix. Indeed, we calculated the best response diagrams for $d=1$ with fixed positive values of $\gamma$
and inferred that suitable values  are  $\beta\ll-1$. Fig.~\ref{cov:mat:estim:fig} presents the average error of the covariance matrix evaluated using the Frobenius norm as a function of $\beta$. The covariance is computed using either the above exact loss function (upper plot) or the statistical approximation of the cumulant loss along with  SGDA algorithm (lower plot) for three values of $\gamma$. We use $10K$ samples for the latter case, average over 10 iterations and a different covariance matrix is used at each iteration. The true covariance matrix is rescaled so that its Frobenius norm equals to 1. We observe that the covariance matrix is learned satisfactorily when the exact loss function is used for large negative values of $\beta$. When the approximated, yet realistic, loss is used, the error between the true and the estimated covariance matrices increases after a certain value of $-\beta$ because tail statistics (requiring a large amount of samples) start to take control. Overall, the direct conclusion is that cumulant GAN is able to learn higher-order statistics and produce samples with the correct covariance structure despite the fact that a very simple discriminator without any access to higher-order statistics was deployed.

\begin{figure}[t!]
\begin{center}
\centerline{\includegraphics[width=0.6\columnwidth]{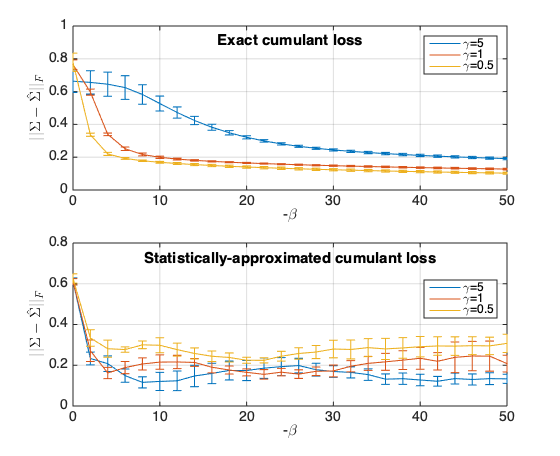}}
\caption{Covariance estimation error for the exact cumulant loss function (upper plot) and for the statistically-approximated cumulant loss function (lower plot).}
\label{cov:mat:estim:fig}
\end{center}
\end{figure}

\subsection{Improved Image Generation}
A series of experiments have been conducted on CIFAR-10 \cite{krizhevsky2009learning} and ImageNet \cite{imagenet_cvpr09} datasets demonstrating the effectiveness of cumulant GAN. In the experiments, we select pairs of $(\beta, \gamma)$ that correspond to well-known divergences in order to highlight their effect on the training process as well as to facilitate connections with existing literature.

\begin{table*}[]\footnotesize
\caption{Mean inception score on CIFAR-10 and ImageNet.}
\hspace{-20mm}
\begin{tabular}{c|cccc||clcl}
& \multicolumn{4}{c||}{CIFAR-10} & \multicolumn{4}{c}{ImageNet} \\ \hline
\backslashbox[]{Loss function}{Architecture} &
\begin{tabular}[c]{@{}c@{}}Conv layers\\ Gen: 3 \& Dis: 1\end{tabular} & 
\begin{tabular}[c]{@{}c@{}}Residual blocks\\ Gen: 4 \& Dis: 2\end{tabular} &
\begin{tabular}[c]{@{}c@{}}Residual blocks (V1)\\ Gen: 4 \& Dis: 3\end{tabular} &
\begin{tabular}[c]{@{}c@{}}Residual blocks (V2)\\ Gen: 4 \& Dis: 3\end{tabular} &
\multicolumn{2}{c}{\begin{tabular}[c]{@{}c@{}}Residual blocks\\ Gen: 4 \& Dis: 2\end{tabular}} &
\multicolumn{2}{c}{\begin{tabular}[c]{@{}c@{}}Residual blocks\\ Gen: 4 \& Dis: 4\end{tabular}} \\ \hline
Wasserstein & $4.36~\pm~0.10$ & $4.58~\pm~0.14$ & $5.25~\pm~0.23$ & $6.45~\pm~0.34$ & \multicolumn{2}{c}{$5.13~\pm~0.45$} & \multicolumn{2}{c}{$8.88~\pm~1.15$} \\
KLD & $\textbf{4.81}~\pm~\textbf{0.07}$ & $7.63~\pm~0.07$ & $\textbf{7.42}~\pm~\textbf{0.08}$ & $7.28~\pm~0.12$ & \multicolumn{2}{c}{$\textbf{8.86}~\pm~\textbf{0.08}$} & \multicolumn{2}{c}{$10.03~\pm~0.12$} \\
Reverse KLD & $4.56~\pm~0.13$ & $\textbf{7.68}~\pm~\textbf{0.08}$ & 7.28~$\pm$~0.05 & $\textbf{7.39}~\pm~\textbf{0.08}$ & \multicolumn{2}{c}{8.70$~\pm~$0.33} & \multicolumn{2}{c}{$\textbf{10.23}~\pm~\textbf{0.13}$} \\
Hellinger & \textbf{4.82}~$\pm$~\textbf{0.10} & \textbf{7.69}~$\pm$~\textbf{0.06} & $7.22~\pm~0.08$ & $\textbf{7.35} \pm \textbf{0.09}$ & \multicolumn{2}{c}{$8.55~\pm~0.23$} & \multicolumn{2}{c}{\textbf{10.24}~$\pm$~\textbf{0.58}} \\ \hline
\end{tabular}
\label{IS:table}
\end{table*}

\begin{table*}[]\footnotesize
\caption{Mean FID on CIFAR-10 and ImageNet.}
\hspace{-20mm}
\begin{tabular}{c|cccc||clcl}
& \multicolumn{4}{c||}{CIFAR-10} & \multicolumn{4}{c}{ImageNet} \\ \hline
\backslashbox[]{Loss function}{Architecture} &
\begin{tabular}[c]{@{}c@{}}Conv layers\\ Gen: 3 \& Dis: 1\end{tabular} & 
\begin{tabular}[c]{@{}c@{}}Residual blocks\\ Gen: 4 \& Dis: 2\end{tabular} &
\begin{tabular}[c]{@{}c@{}}Residual blocks (V1)\\ Gen: 4 \& Dis: 3\end{tabular} &
\begin{tabular}[c]{@{}c@{}}Residual blocks (V2)\\ Gen: 4 \& Dis: 3\end{tabular} &
\multicolumn{2}{c}{\begin{tabular}[c]{@{}c@{}}Residual blocks\\ Gen: 4 \& Dis: 2\end{tabular}} &
\multicolumn{2}{c}{\begin{tabular}[c]{@{}c@{}}Residual blocks\\ Gen: 4 \& Dis: 4\end{tabular}} \\ \hline
Wasserstein & $173.66~\pm~3.42$ & $76.54~\pm~4.04$ & $71.77~\pm~3.63$ & $39.58~\pm~13.85$ & \multicolumn{2}{c}{$137.78~\pm~11.20$} & \multicolumn{2}{c}{$64.26~\pm~16.06$} \\
KLD & $\textbf{157.36}~\pm~\textbf{2.02}$ & $19.81~\pm~0.55$ & $\textbf{22.51}~\pm~\textbf{1.75}$ & $23.06~\pm~1.20$ & \multicolumn{2}{c}{$\textbf{67.45}~\pm~\textbf{1.74}$} & \multicolumn{2}{c}{$49.39~\pm~0.75$} \\
Reverse KLD & $\textbf{156.23}~\pm~\textbf{6.51}$ & $\textbf{18.74}~\pm~\textbf{0.58}$ & $24.62~\pm~0.90$ & $\textbf{21.54}~\pm~\textbf{1.36}$ & \multicolumn{2}{c}{$72.96~\pm~5.57$} & \multicolumn{2}{c}{$\textbf{45.91}~\pm~\textbf{1.40}$} \\
Hellinger & $\textbf{158.60}~\pm~\textbf{2.96}$ & $\textbf{18.66}~\pm~\textbf{0.54}$ & $24.59~\pm~0.63$ & $\textbf{21.15}~\pm~\textbf{1.25}$ & \multicolumn{2}{c}{$69.88~\pm~1.98$} & \multicolumn{2}{c}{$\textbf{45.80}~\pm~\textbf{3.77}$} \\ \hline
\end{tabular}
\label{FID:table}
\end{table*}

\subsubsection{CIFAR-10 Dataset}
CIFAR-10 is a well-studied dataset of 32$\times$32$\times$3 RGB color images with 10 classes. We evaluate the quality of the generated images using four different architectures: one with convolutional layers (CNN) and three with residual blocks (resnet). The generator for the CNN consists of one linear layer followed by three convolutional layers while the discriminator is a single convolutional layer followed by one linear layer. The generator for the three resnets consists of four residual blocks while the discriminator consists of two or three residual blocks. We train two versions with three residual blocks for the discriminator but with different channel dimension and learning rate. The complete description of the architectures can be found in Appendix F of Supplementary Materials. In all cases, we deliberately choose a weaker discriminator to challenge the training procedure.

Tables \ref{IS:table} and \ref{FID:table} report the averaged inception score (IS) \cite{salimans2016improved} and the averaged Fr\'echet inception distance (FID) \cite{heusel2017gans} along with their standard deviation over five runs for the four architectures. We test four different hyper-parameter values that correspond to minimization of Wasserstein distance, KLD, reverse KLD and Hellinger distance (actually, $-4\log(1-Hel^2)$). 
We use the two-sided gradient penalty for WGAN since it has been shown to provide better performance than the one-sided version \cite{gulrajani2017improved}. However, the two-sided gradient penalty is not valid for cumulant GAN \cite{birrell2021fgammadivergences} therefore we enforce the one-sided version of the gradient penalty. In all cases, the optimization for the discriminator is realized over Lipschitz continuous functions.
The implementation of cumulant GAN is based on available open-source code\footnote{https://github.com/igul222/improved\_wgan\_training}. Following the reference code, we train the models with the Adam optimizer and the discriminator's parameters are updated $k=5$ times more often than the parameters of the generator.

We remind that IS is a standard metric to evaluate the visual quality of generated image samples \cite{salimans2016improved}. IS correlates with human judgment by feeding generated samples into a pre-trained Inception v3 classifier. Images with naturally-looking objects are supposed to have low label (output) entropy whereas the entropy across images should be high. On the other hand, FID score uses the Inception v3 model activation layers (last pooling layer) to capture latent features calculated for a collection of real and generated images. The activation values are summarized as a multivariate Gaussian by calculating the mean and covariance of both real and generated images. The distance between these two distributions is then calculated using the Fr\'echet distance, also called the 2-Wasserstein distance. We use $50K$ images to compute IS/FID scores. Higher IS means better generated image quality whereas the best generative model returns the lowest FID score.

We observe from the tables as well as from the panels of Fig.~\ref{cifar:fig} which present the averaged IS as a function of the number of iterations for the four architectures that all hyper-parameter choices for cumulant GAN outperform the baseline WGAN. The relative improvement ranged from 4.6\% (reverse KLD) up to 10.5\% (Hellinger distance) for the CNN architecture while the relative improvement for the resnet with the weaker discriminator ranged from 66.6\% (reverse KLD) up to 67.9\% (Hellinger distance) revealing that cumulant GAN takes into consideration all discriminator's moments, i.e., all higher-order statistics and not just the mean values, leading to better realization of the target distribution. Cumulant GAN achieves higher ISs than WGAN for the two versions of resnets with three residual blocks for the discriminator (lower panels in Fig. \ref{cifar:fig}), too. As expected, cumulant GANs also perform better on FID metric with relative improvements up to 75.62\% (Hellinger distance) for the resnet with the weaker discriminator and 68.64\% (KLD) and 46.56\% (Hellinger distance) for the two versions of resnets, respectively. All cumulant GAN variations (KLD, reverse KLD and Hellinger) obtain similar results for both versions while the performance of WGAN is significantly affected by the choice of the hyper-parameter values, e.g., learning rate and channel dimension. This discrepancy in the performance highlights the enhanced robustness of cumulant GAN relative to WGAN implying that cumulant GAN may require less tuning in order to enjoy excellent performance. Finally, the samples generated by cumulant GAN also exhibit larger diversity and are visually better (we refer to Appendix G in Supplementary Materials).

\begin{figure}[t!]
\begin{center}
\centerline{\includegraphics[width=.85\columnwidth]{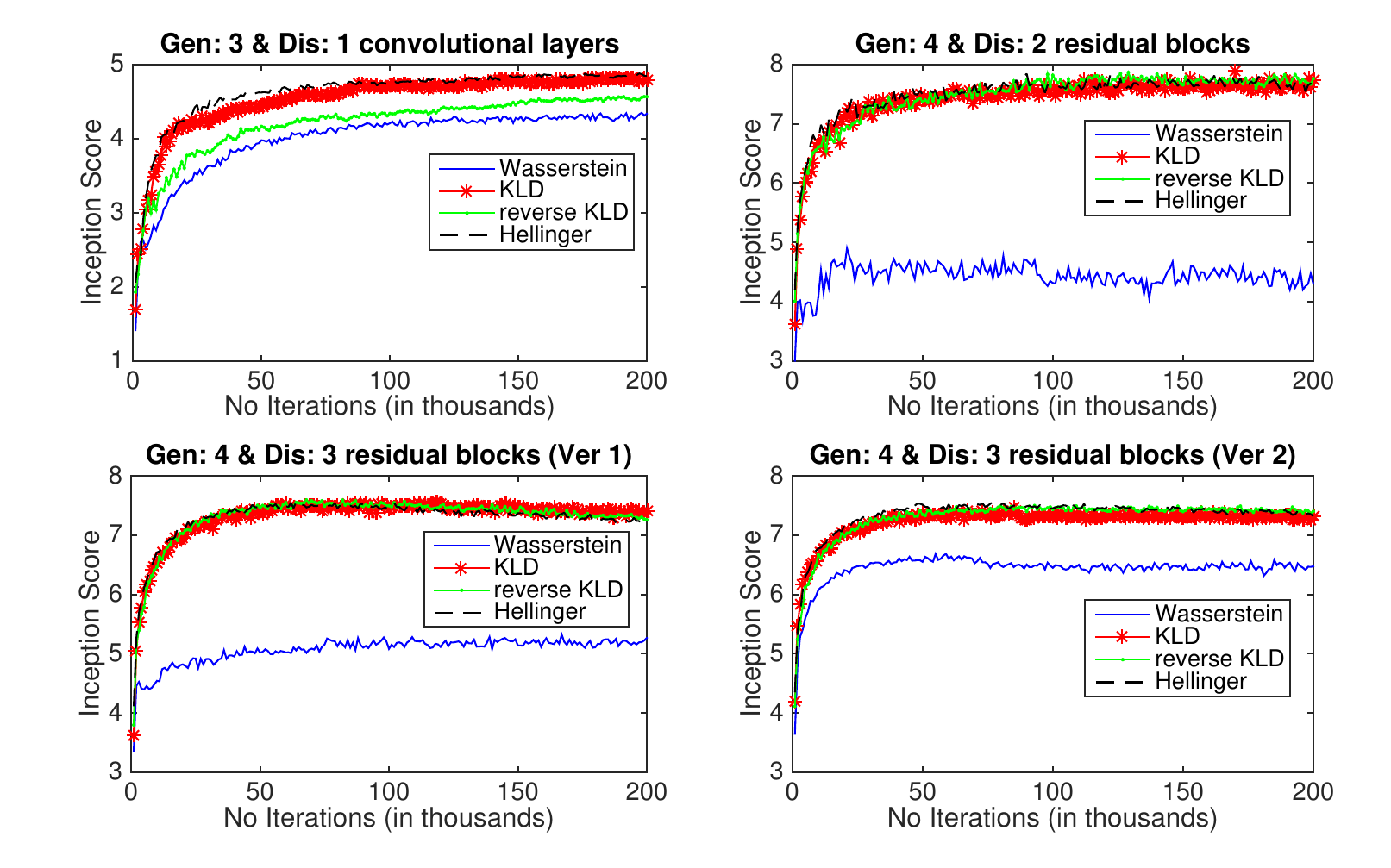}}
\caption{IS for CIFAR-10 using various hyper-parameters of cumulant GAN and various architectures. In all cases, WGAN has a lower IS relative to the cumulant GAN with the hyper-parameter corresponding to Hellinger minimization achieving the best overall performance.} 
\label{cifar:fig}
\end{center}
\end{figure}

\subsubsection{ImageNet Dataset}
This large dataset consists of 64$\times$64$\times$3 color images with 1000 object classes. The large number of classes is challenging for GAN training due to the tendency to underestimate the entropy in the distribution \cite{salimans2016improved}. We evaluate the performance on two different architectures which both have a generator with four residual blocks. The difference is in the number of residual blocks for the discriminator where we employ a weak discriminator with two residual blocks and a strong discriminator with four residual blocks.
Fig.~\ref{imagenet:fig} presents the performance in terms of IS both for the baseline WGAN and for the variants of cumulant GAN when a weak discriminator (left panel) or a strong discriminator (right panel) is utilized. Improved inception scores are obtained with cumulant GAN for both architectures. It is also important to note that our approach is much more effective than WGAN at avoiding mode collapse while still generating high quality samples. The mean ISs along with the standard deviation over three repetitions are reported in Table \ref{IS:table} (rightmost columns). In terms of relative improvement, cumulant GAN is between 72.71\% (KLD) to 69.59\% (reverse KLD) better than WGAN for the weak discriminator and a similar trend is observed when the strong discriminator is used reaching IS as high as $10.24$. As reported in Table \ref{FID:table}, the proposed cumulant GAN is superior relative to WGAN in generating high quality images with low FID scores of 67.45 (KLD) for the weak discriminator and 45.80 (Hellinger distance) for the strong discriminator.
By visual inspection of the generated images (Appendix G in Supplementary Materials), we conclude that all generators learn some basic and contiguous shapes with natural color and texture. Nevertheless, cumulant GAN provides better images with object specifications that are clearly more realistic.  

\begin{figure}[t!]
\begin{center}
\centerline{\includegraphics[width=.85\columnwidth]{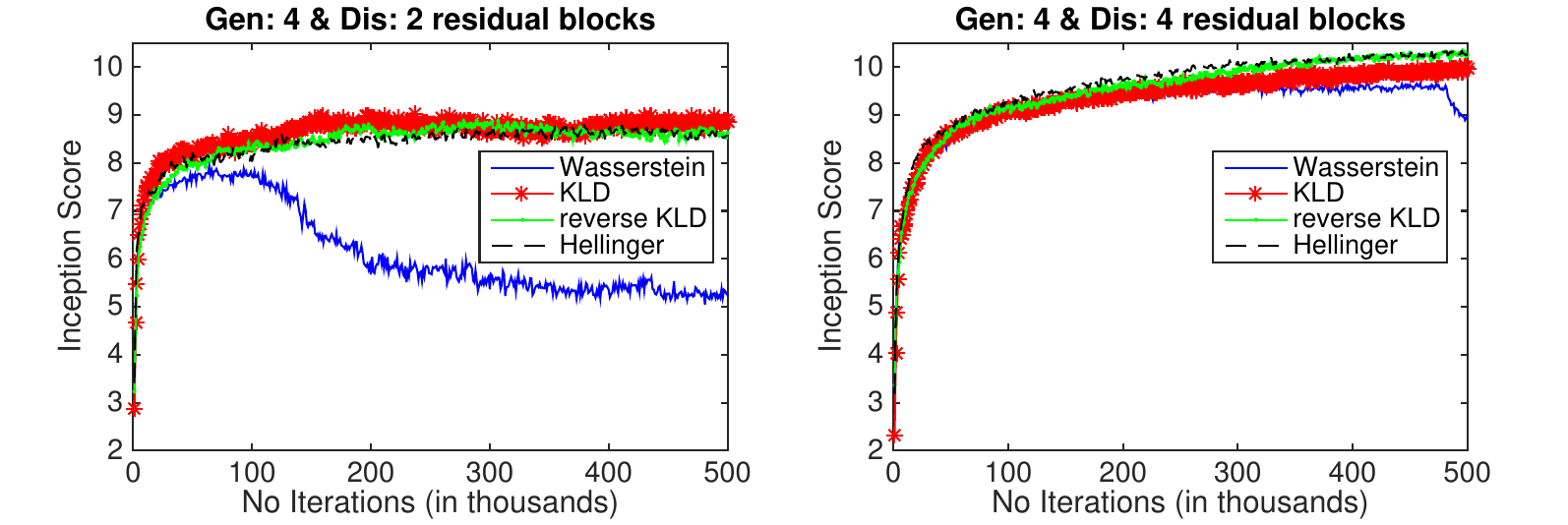}}
\caption{Same as Fig. \ref{cifar:fig} but for ImageNet. Cumulant GAN achieves higher IS relative to WGAN for both weak (left panel) and strong (right panel) discriminator. Mode collapse has been mitigated in all cumulant GAN variants.}
\label{imagenet:fig}
\end{center}
\end{figure}

Despite not being exhaustive, the presented examples demonstrated a preference of cumulant GAN over WGAN. In general, GAN optimization has essentially two critical components: the first being the function space where the discriminator lives while the other is the objective functional to be optimized. WGAN's breakthrough was the restriction of the function space to Lipschitz continuous functions that resulted in increased stability. However, there is no evidence that the best-performing loss function is the difference of two expectations as in WGAN. The presented examples revealed that there are better and more flexible options for the overall loss function and the proposed cumulant loss is one of them. 

\section{Conclusions and Future Directions}
\label{concl:sec}
We proposed cumulant GAN by establishing a novel loss function based on the CGF of both real and generated distributions. The use of CGFs allows for an inclusive   characterization of the distributions' statistics, making it possible to partially remove complexity from the discriminator. The net result is improved and more stable training of GANs. Furthermore, cumulant GAN has the capacity to smoothly interpolate between a wide range of divergences and distances by simply changing  its two hyper-parameter values $\beta$ \& $\gamma$. Thus it  offers a flexible and comprehensive mechanism to choose --possibly adaptively-- which objective to minimize. In a recent publication \cite{dip2021disentangle}, the authors applied cumulant GAN for disentangled representation learning of speech signals and our plan is to further explore the improved capabilities of cumulant GAN in a variety of estimation and inference applications.
Finally, the substitution of an expectation operator with the respective CGF does not have to be limited to WGAN. It can be applied to other GANs' loss function resulting in new GAN formulations. The theoretical and empirical ramifications of such extensions are left as future work.

\section*{Acknowledgments}
Yannis Pantazis acknowledges partial support by the project “Innovative Actions in Environmental Research and Development (PErAn)” (MIS 5002358) funded by the Operational Programme "Competitiveness, Entrepreneurship and Innovation" (NSRF 2014-2020). 
Dipjyoti Paul has received for this work funding from the EU H2020 Research and Innovation Programme under the MSCA GA 67532 (the ENRICH network: www.enrich-etn.eu).
Dr. Michail Fasoulakis is supported by the Stavros Niarchos-FORTH postdoc fellowship for the project ARCHERS. 
The research of Markos Katsoulakis was partially supported by  the National Science Foundation (NSF) under the grant DMS-2008970, by the HDR-TRIPODS program of NSF under the grant CISE-1934846 and by the Air Force Office of Scientific Research (AFOSR) under the grant FA-9550-18-1-0214.

\bibliographystyle{unsrt}  
\bibliography{bibl.bib}

\clearpage

\appendix

\section{A Variational Formula for Renyi Divergence}
Similarly to the Donsker-Varahdan variational formula for the Kullback-Leibler divergence that can be obtained from the convex duality formula, we prove a variational formula for the Renyi divergence using the variational representation of exponential integrals also known as  risk-sensitive functionals/observables.

\begin{theorem} {\bf (Variational Representation of Renyi Divergences)}
Let $p$ and $q$ be probability distributions. Then, the following formula  holds: 
\begin{equation}
\mathcal R_\alpha(p||q) = \sup_{f\in C_b} \left\{ \frac{1}{\alpha-1} \log \mathbb E_p[e^{(\alpha-1)f}] - \frac{1}{\alpha} \log \mathbb E_q[e^{\alpha f}] \right\},
\label{Renyi:var:eq}
\end{equation}
where $C_b$ is the space of all bounded and measurable functions.
\end{theorem}

\begin{proof}
The authors in \cite{dupuis:renyi} proved that for all bounded and measurable functions $f$ we have:
$$
\frac{1}{\alpha-1} \log \mathbb E_p[e^{(\alpha-1)f}] = \inf_q \{ \frac{1}{\alpha} \log \mathbb E_q[e^{\alpha f}] + \mathcal R_\alpha(p||q) \}.
$$

Therefore, for any $q$, 
$$
\frac{1}{\alpha-1} \log \mathbb E_p[e^{(\alpha-1)f}] \le \frac{1}{\alpha} \log \mathbb E_q[e^{\alpha f}] + \mathcal R_\alpha(p||q)
$$
$$
\mathcal R_\alpha(p||q) \ge \frac{1}{\alpha-1} \log \mathbb E_p[e^{(\alpha-1)f}] - \frac{1}{\alpha} \log \mathbb E_q[e^{\alpha f}]
$$
For simplicity in the presentation, here  we provide the proof based on the assumption that the function $f=\log\frac{dp}{dq}$ is bounded and measurable. 
Based on the aforementioned assumption we have: 
\begin{equation*}
\begin{aligned}
&\frac{1}{\alpha-1} \log \mathbb E_p[e^{(\alpha-1)\log\frac{dp}{dq}}] - \frac{1}{\alpha} \log \mathbb E_q[e^{\alpha \log\frac{dp}{dq}}]\\
=& \frac{1}{\alpha-1} \log \mathbb E_q[(\frac{dp}{dq})^\alpha] - \frac{1}{\alpha} \log \mathbb E_q[(\frac{dp}{dq})^\alpha] \\
=& \frac{1}{(\alpha-1)\alpha} \log \mathbb E_q[(\frac{dp}{dq})^\alpha] \\
=& \mathcal R_\alpha(p||q)
\end{aligned}
\end{equation*}

Therefore, the supremum is attained hence we proved (\ref{Renyi:var:eq}).
We refer to \cite{birrell2019distributional} for the complete and general proof. 
\end{proof}
It is not hard to show that the variational formula for Renyi divergence reduces to the well-known Donsker-Varahdan variational formula for the Kullback-Leibler divergence, when $\alpha \to 1$, \cite{birrell2019distributional}.

\section{Concavity Property of Cumulant GAN}
The concavity of the logarithmic function implies that 
\begin{equation}
\beta^{-1} \Lambda_{f,p} (\beta) \ge \mathbb E_p[ f(x)] \ ,
\label{concave:cgf:eq}
\end{equation}
which is nothing else but Jensen's inequality. If additionally $f$ is bounded, i.e., there is $M>0$ such that $|f(x)|\le M$ for all $x$ then a stronger inequality is obtained due to the fact that the domain of the logarithm is also bounded. Indeed, logarithm is strongly concave with modulus equal to the infimum value of the domain. In our case, strongly Jensen's inequality deduces that 
\begin{equation}
\beta^{-1} \Lambda_{f,p} (\beta) \ge \mathbb E_p[ f(x)] - \beta e^{-\beta M} \var_p(f(x))
\label{strongly:concave:cgf:eq}
\end{equation}

From Jensen's inequality (\ref{concave:cgf:eq}), it is easy to show that for all $\beta,\gamma\neq 0$
\begin{equation}
L(\beta,\gamma) \ge L(0,0) = \mathbb E_{p_r}[D(x)] - \mathbb E_{p_g}[D(x)]
\label{ineq:cumgan:loss:eq}
\end{equation}
A stricter inequality called Jensen's inequality for strongly convex/concave functions can be obtained is the function $D$ is bounded. Indeed, if $|D(x)|<M$ for all $x$ then the domain of the logarithmic function is also bounded leading to the stronger inequality
\begin{equation}
% \begin{split}
L(\beta,\gamma) \ge \mathbb E_{p_r}[D(x)] - \mathbb E_{p_g}[D(x)] 
- \beta e^{-\beta M} \var_{p_r}(D(x))  
- \gamma e^{-\gamma M} \var_{p_g}(D(x)) .
\label{strong:ineq:cumgan:loss:eq}
% \end{split}
\end{equation}
Generally speaking, strong concavity/convexity is a strengthening of the notion of concavity/convexity, and some properties of strongly concave/convex functions are just “stronger versions” of analogous properties of concave/convex functions.

\section{A Quantitative Demonstration of the Weights' Interpretation}

We present an explicit example and highlight the interpretation of cumulant GAN as a weighing mechanism of each sample gradient contribution. We consider a one-dimensional Gaussian with mean value $\mu=1$ and variance $\sigma^2=1$, a linear generator $G(z)=z+\theta$ with $z\sim\mathcal{N}(0,1^2)$ and a linear discriminator $D(x) = \eta x$. We draw $m=10$ samples from the real distribution and $m=10$ samples from the noise distribution which are kept fixed for visualization purposes throughout the training phase.

\begin{figure}[ht]
\begin{center}
\includegraphics[width=0.9\columnwidth]{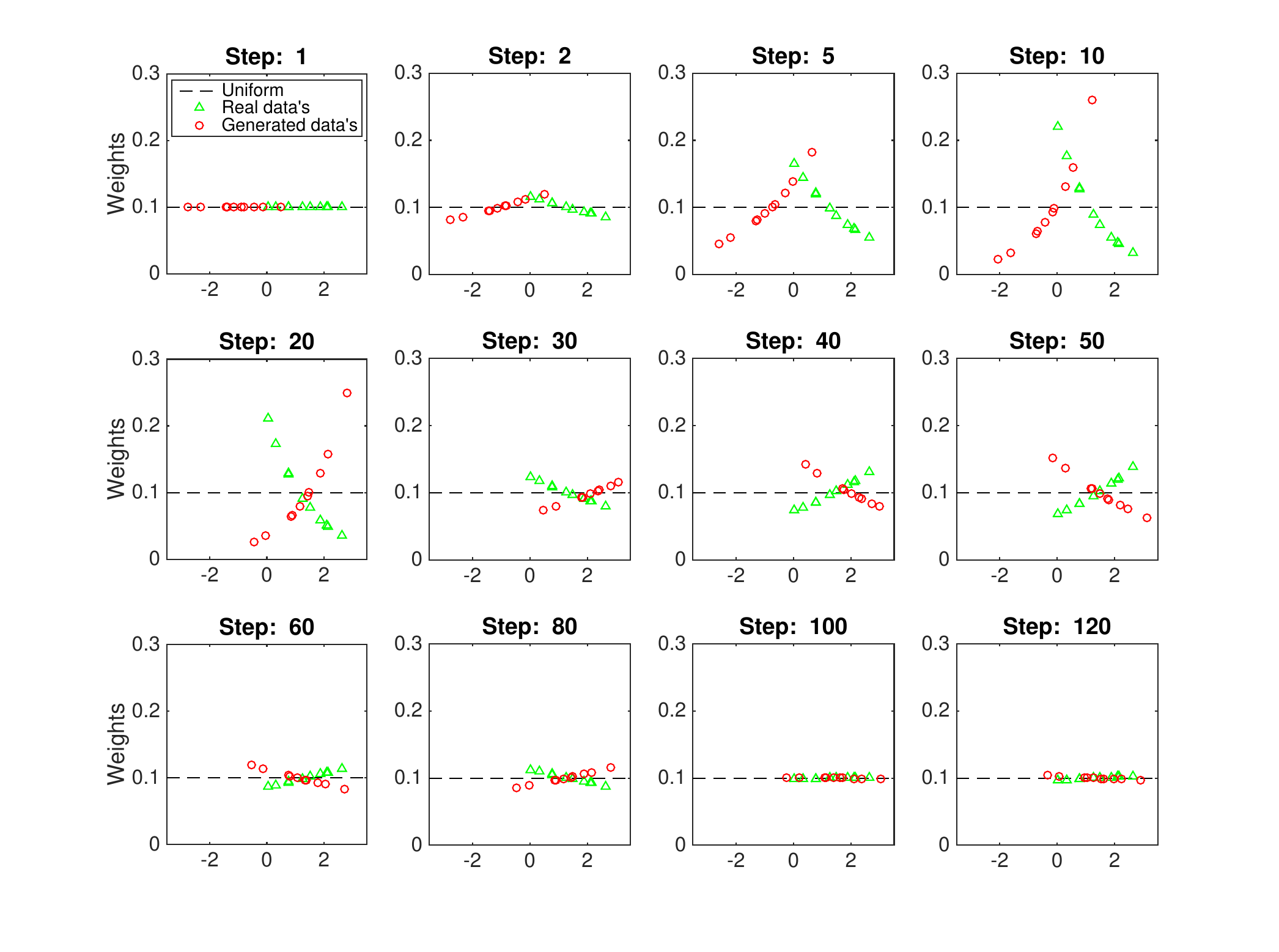}
\caption{The evolution of the weights $w_i^\beta={e^{-\beta D\left(x_i\right)}}/{\sum_{j=1}^m e^{-\beta D\left(x_j\right)}}$ (green triangles) and $w_i^\gamma={e^{\gamma D\left(G(z_i)\right)}}/{\sum_{j=1}^m e^{\gamma D\left(G(z_j)\right)}}$ (red circles) with $i=1,...,m$ as training progresses. Here, we set $\beta=\gamma=0.5$ while one update for the discriminator is followed by one update for the generator (i.e., $k=1$ in Algorithm 1 of the main text).}
\label{weights:k=1:fig}
\end{center}
\end{figure}

\begin{figure}[t]
\begin{center}
\includegraphics[width=0.9\columnwidth]{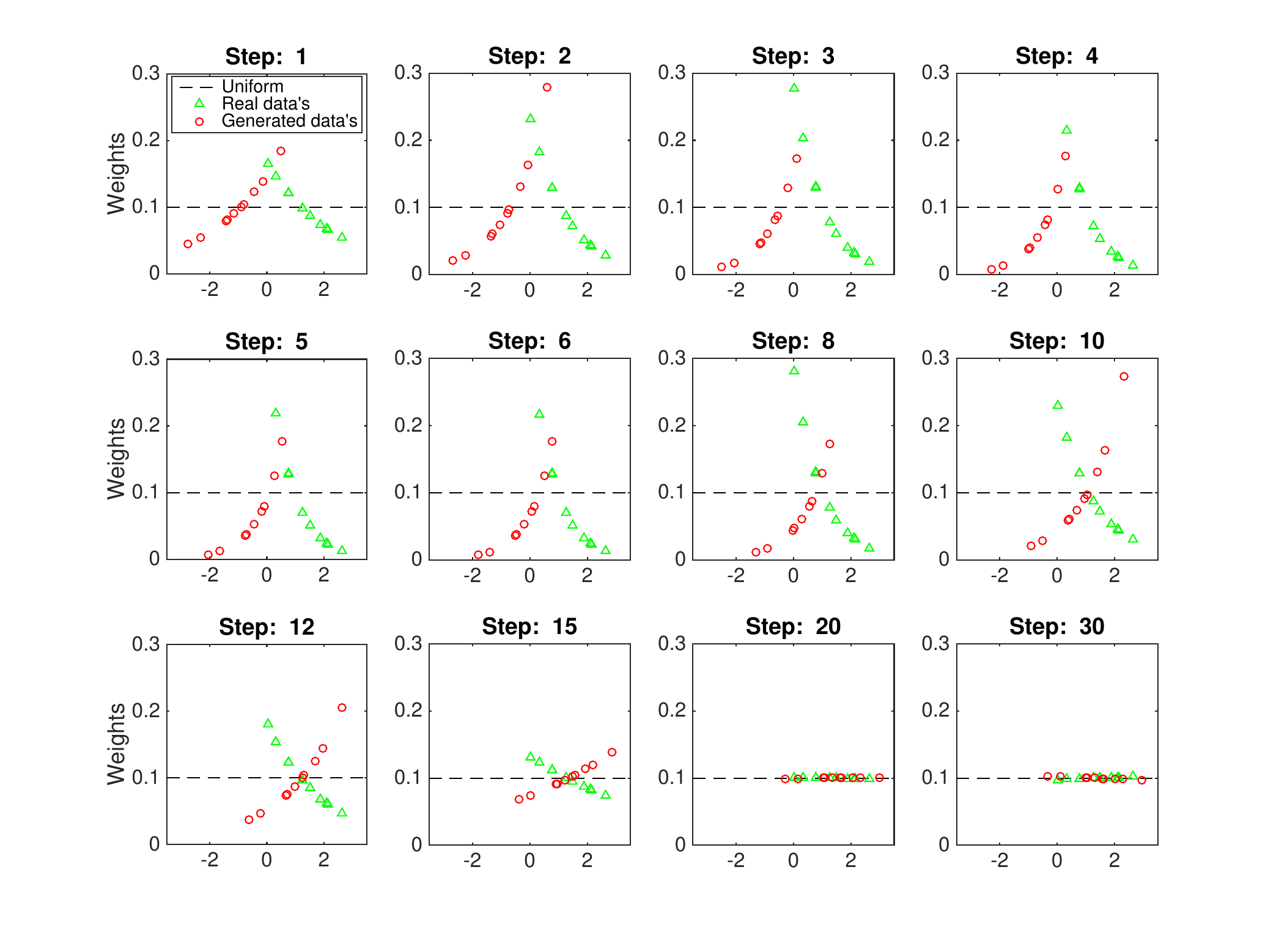}
\caption{Same as Fig. \ref{weights:k=1:fig} but with $k=5$ gradient updates for the discriminator before a single update for the generator. Now, the optimization algorithm converges about 5 times faster relative to the case where $k=1$ as evidenced by a direct comparison in the number of required steps.}
\label{weights:k=5:fig}
\end{center}
\end{figure}

\begin{figure}[t]
\begin{center}
\includegraphics[width=0.42\columnwidth]{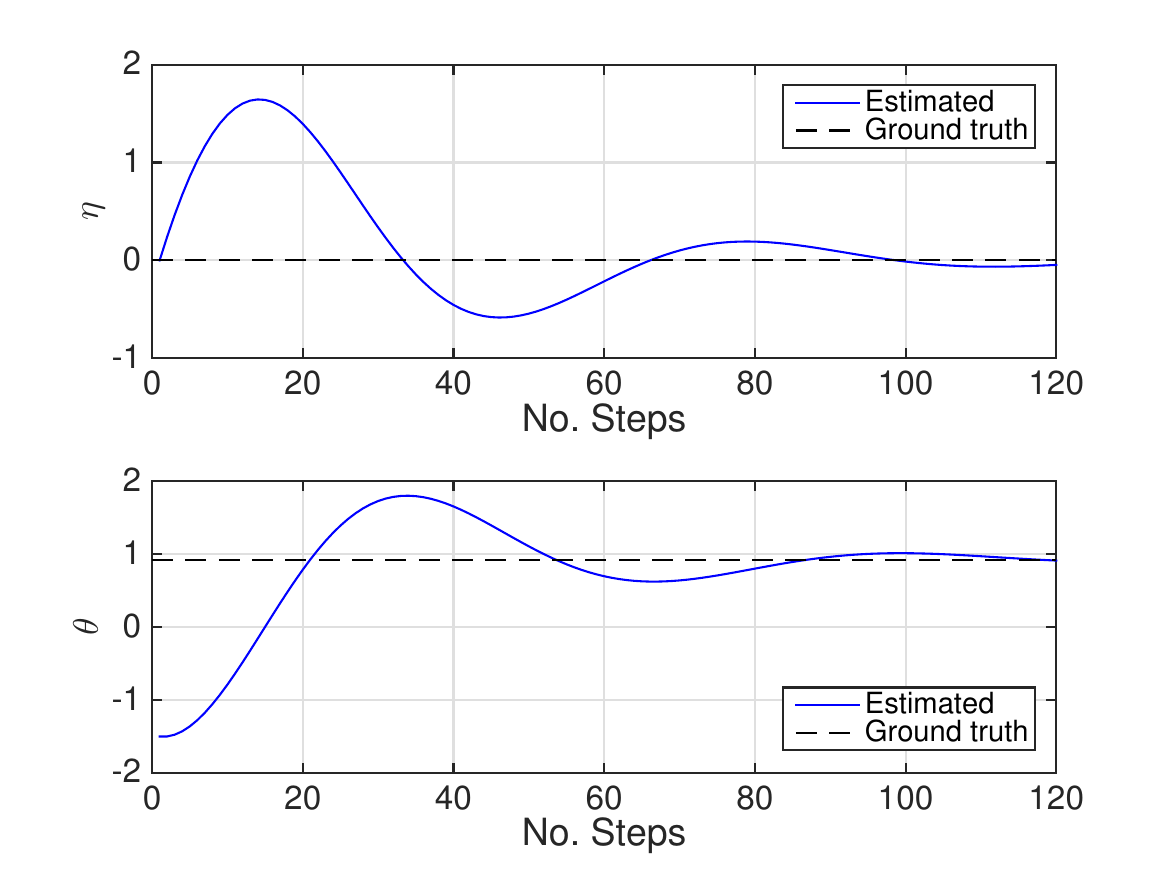}
\includegraphics[width=0.42\columnwidth]{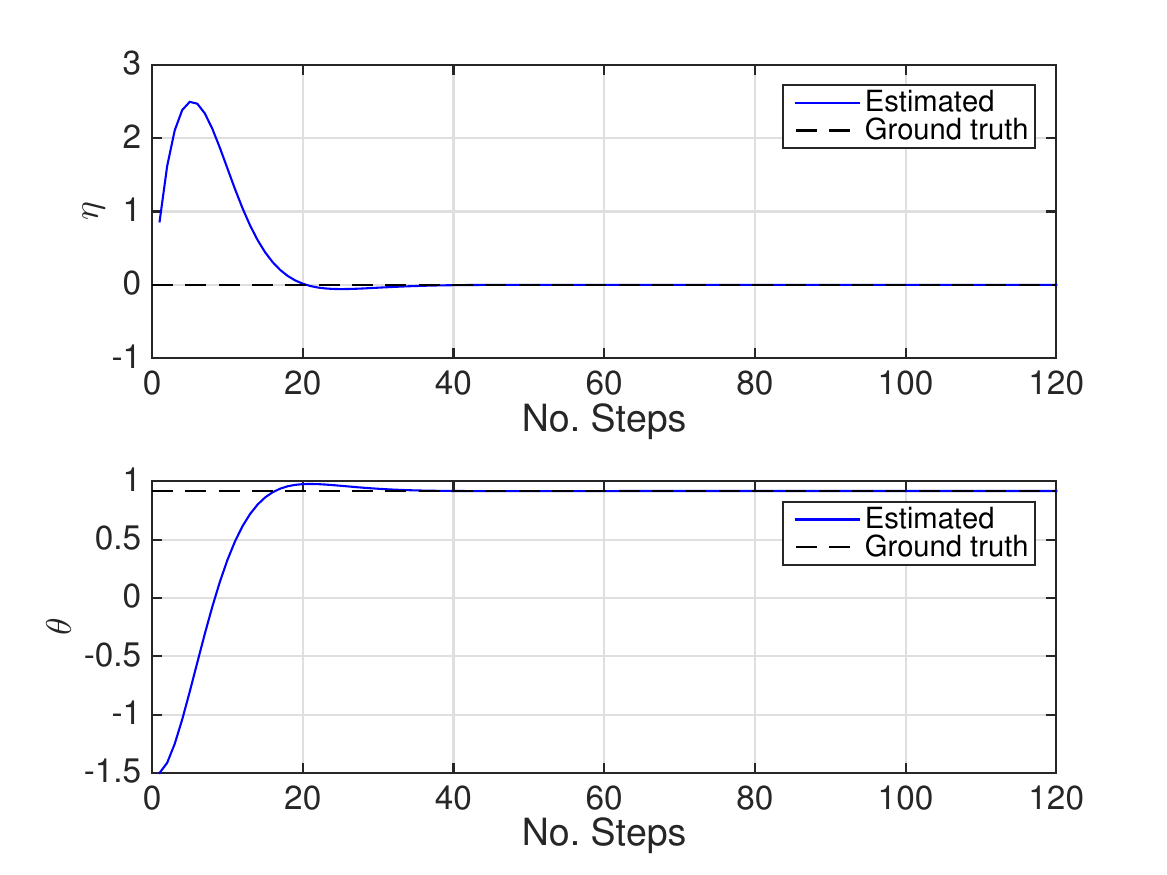}
\caption{ The evolution of the estimated parameters for the discriminator (upper plots) and the generator (lower plots) as a function of the number of training steps. The leftmost plots correspond to $k=1$ while $k=5$ in the rightmost plots. We clarify that the ground truth (dashed lines) corresponds to the statistical mean of the realized real data (i.e., to $\hat{\mu}=\frac{1}{m}\sum_{i=1}^m x_i$).}
\label{parameter:dynamics:fig}
\end{center}
\end{figure}

Fig. \ref{weights:k=1:fig} shows the weights $w_i^\beta$ (green triangles) and $w_i^\gamma$ (red circles) for each sample of the real and synthetic distribution, respectively,  at various stages of the training procedure. During the first training steps (first row of plots in Fig. \ref{weights:k=1:fig}), we observe that the samples from both distributions that are close to the decision boundary or inside the wrong classification regime are assigned with higher weights (see Steps 5 \& 10). Indeed, a positive $\eta$ made the samples close to the decision boundary to have larger weights and thus contribute more to the gradient update. As training progresses (steps between 10 and 30), the red cloud of points (the synthetic data) moves beyond the green cloud of points implying that the estimated parameter $\theta$ advances far  further to the right surpassing the true mean thus the samples that previously caused trouble to the discriminator becomes now easy to discriminate because the discriminator switched the sign for its parameter, $\eta$. Subsequently, the decision boundary changed to the opposite end of the distributions' support and the respective samples enjoy now higher weights (step 40). This flipping is evident between steps 30 and 40 as well as between 60 and 80. As expected, the weights become uniformly equal at equilibrium when the optimization algorithm has converged to the optimum solution (training steps beyond 100).

The evolution of the estimated parameters as a function of the training steps is shown in Fig. \ref{parameter:dynamics:fig} (leftmost plots). The (damped) oscillatory behavior can be further reduced by applying the gradient update for the discriminator $k$ times before an update of the generator's parameters. The rightmost plots of Fig. \ref{parameter:dynamics:fig} shows the evolution of the estimated parameters when $k=5$. Additionally, Fig. \ref{weights:k=5:fig} demonstrates the evolution of the weights for $k=5$. Similar conclusions as for $k=1$ can be made for the weights, however, it is evident that a more faithful (and better) discriminator made the weights more accurate and reliable and thus the training process converged significantly faster.
Overall, this example, despite being simple, highlights the effect of the discriminator on the weights as well as quantitatively demonstrates the interpretation of cumulant GAN as a weighting mechanism for $\beta,\gamma>0$.

\section{Core Code for Cumulant GAN}
Next, we present the core part of the implementation of \emph{cumulant GAN}.

{\small
\begin{lstlisting}[language=Python]
fake_data = Generator()
disc_real = Discriminator(real_data)
disc_fake = Discriminator(fake_data)

def loss_function(disc_real,disc_fake,beta,gamma):


    max_val = tf.reduce_max((-beta) * disc_real)
    disc_cost_real = -(1.0/beta)*(tf.log(tf.reduce_mean(tf.exp((-beta)*disc_real-max_val)))+max_val)

    max_val = tf.reduce_max((gamma) * disc_fake)
    disc_cost_fake = (1.0/gamma)*(tf.log(tf.reduce_mean(tf.exp(gamma*disc_fake-max_val)))+max_val)
    gen_cost = -(1.0/gamma)*(tf.log(tf.reduce_mean(tf.exp(gamma*disc_fake-max_val)))+max_val)

    disc_cost = disc_cost_fake - disc_cost_real
    
    alpha = tf.random_uniform(
        shape=[64,1], 
        minval=0.,maxval=1.)
        
    differences = fake_data - real_data
    interpolates = real_data + (alpha*differences)
    gradients = tf.gradients(Discriminator(interpolates), [interpolates])[0]
    slopes = tf.sqrt(tf.reduce_sum(tf.square(gradients),reduction_indices=[1]))
    gradient_penalty = tf.reduce_mean((tf.math.maximum(0.,(slopes-1.)))**2)
    disc_cost += 10*gradient_penalty
    
    gen_train_op = tf.train.AdamOptimizer(learning_rate=1e-4, beta1=0.,
        beta2=0.9).minimize(gen_cost,
        var_list=lib.params_with_name('Generator'), 
        colocate_gradients_with_ops=True)
        
    disc_train_op = tf.train.AdamOptimizer(learning_rate=1e-4, beta1=0.,
        beta2=0.9).minimize(disc_cost,
        var_list=lib.params_with_name('Discriminator.'), 
        colocate_gradients_with_ops=True)
 
    return gen_train_op, disc_train_op
    
\end{lstlisting}}

\section{Additional Experiments with Synthetic Datasets}

\subsection{GMM-8 dataset with different neural network architectures}
We first assess the performance of cumulant GAN using different architectures for the generator and the discriminator. The dataset is the GMM with 8 equiprobable modes as in the main text. We tested three architectures which is the same for both the generator and the discriminator. These are: 2 hidden layers with 64 and 16 units per layer (Figs. \ref{2layers:64units:fig} and \ref{2layers:16units:fig}, repsectively) and 3 hidden layers with 32 units per layer (Fig. \ref{3layers:32units:fig}). Evidently, the minimization of Hellinger distance gave the best results (last column in each figure) in terms of accuracy. Despite the performance differences, the main conclusion that KLD performs mode covering while reverse KLD performs mode selection is still valid. 

\begin{figure}[ht]
\begin{center}
\centerline{\includegraphics[width=0.95\columnwidth]{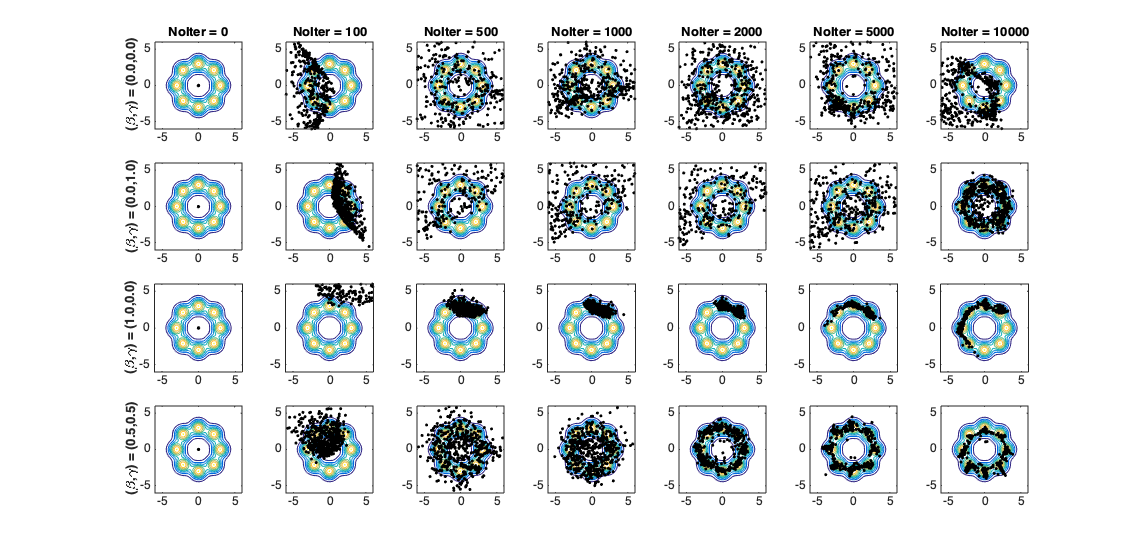}}

\caption{Generated samples using the Wasserstein distance using clipping (1st row), KL divergence (2nd row), reverse KLD (3rd row) and Hellinger distance (last row). Two hidden layers with 64 units per layer are used. The results are similar to Fig. 3 from the main text.}
\label{2layers:64units:fig}
\end{center}
\end{figure}

\begin{figure}[h]
\begin{center}
\centerline{\includegraphics[width=0.9\columnwidth]{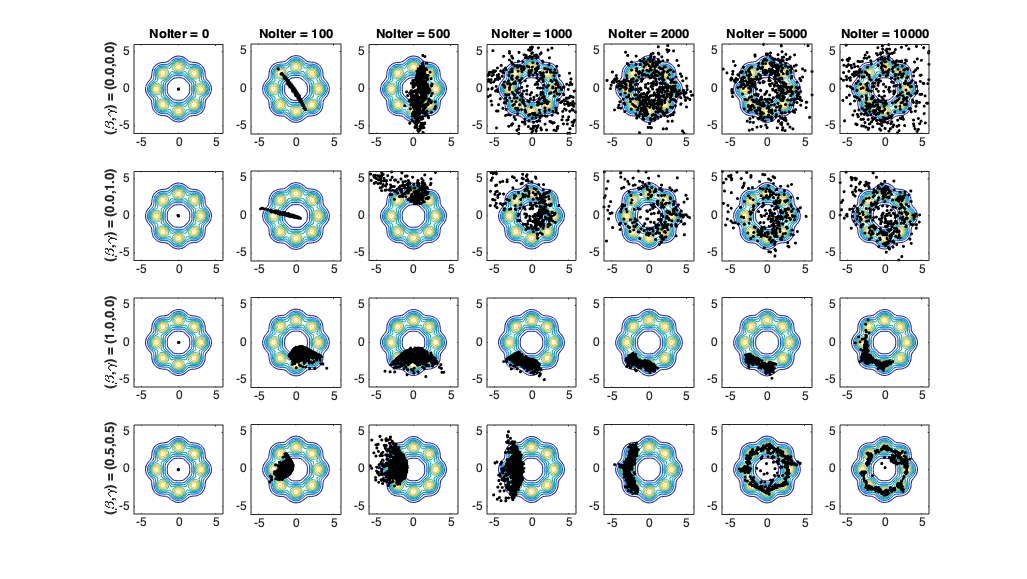}}
\caption{Same as Fig. \ref{2layers:64units:fig} but with 16 units per hidden layer. The capacity of the neural nets is low resulting in convergence instabilities. Despite the low capacity, training with Hellinger distance (last row) resulted in an accurate and stable solution.}
\label{2layers:16units:fig}
\end{center}
\end{figure}

\begin{figure}[h]
\begin{center}
\centerline{\includegraphics[width=0.9\columnwidth]{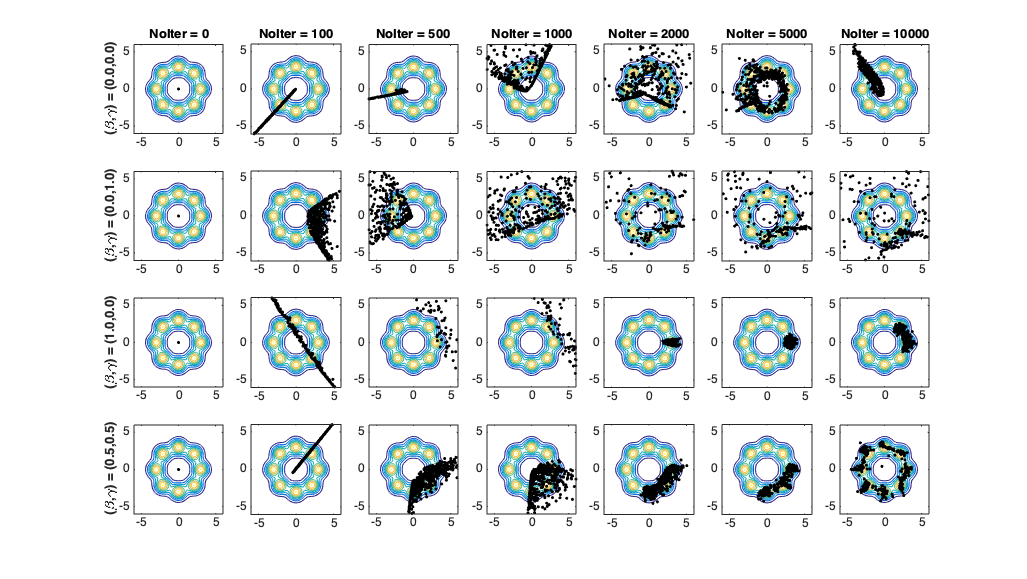}}
\caption{Same as Fig. \ref{2layers:64units:fig} but with 3 hidden layers and 32 units per layer. Here the capacity of the NNs, especially for the discriminator, is rather high resulting in gradient vanishing. Again, Hellinger distance (last row) produced the most accurate result for the given number of iterations.}
\label{3layers:32units:fig}
\end{center}
\end{figure}

\newpage

\subsection{Testing Cumulant GAN on various datasets}
We consider two synthetic datasets with diverse statistical properties to further validate the outcomes of cumulant GAN. The first one is a mixture model of 6 equiprobable Student's t distributions. The characteristic property of this distribution is that it is heavy-tailed. Thus samples can be observed far from the mean value. Figs. \ref{tmm6:32units:fig}-\ref{tmm6:32units:3layers:fig} present the learning dynamics for three different architectures. We observe that the choice of the architecture is crucial, especially when KLD is utilized while Hellinger distance minimization provided the most robust and accurate results across all tested architectures. On the other hand, WGAN was never able to converge when the number of iterations is set to $10K$.

% \subsubsection{Mixture of heavy-tailed distributions}

\begin{figure}[ht]
\begin{center}
\centerline{\includegraphics[width=0.9\columnwidth]{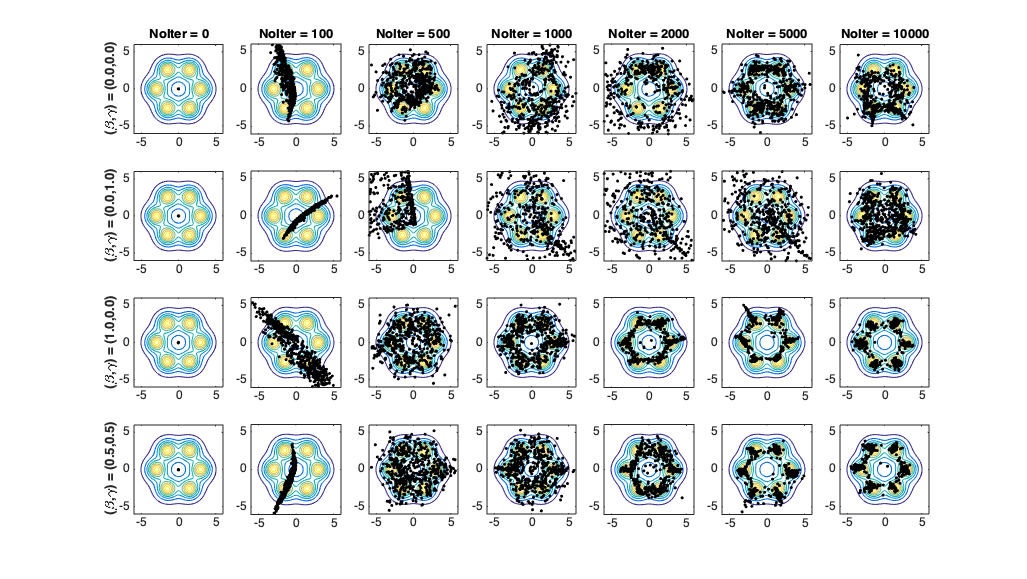}}
\caption{Generated samples using the Wasserstein distance using clipping (1st row), KL divergence (2nd row), reverse KLD (3rd row) and Hellinger distance (last row). 2 hidden layers with 32 units per layer are used. KLD minimization diverged for this particular architectural choice.}
\label{tmm6:32units:fig}
\end{center}
\end{figure}

\begin{figure}[ht]
\begin{center}
\centerline{\includegraphics[width=0.95\columnwidth]{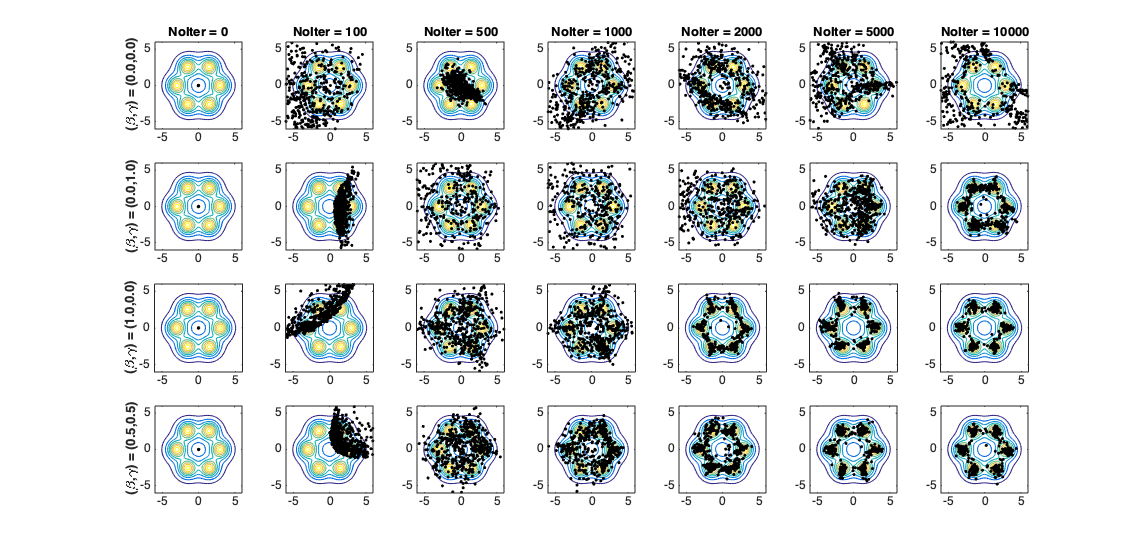}}
\caption{Same as Fig. \ref{tmm6:32units:fig} but with 64 units per layer. This architectural choice provided the most accurate results across all cumulant loss functions.}
\label{tmm6:64units:fig}
\end{center}
\end{figure}

\begin{figure}[h]
\begin{center}
\centerline{\includegraphics[width=0.95\columnwidth]{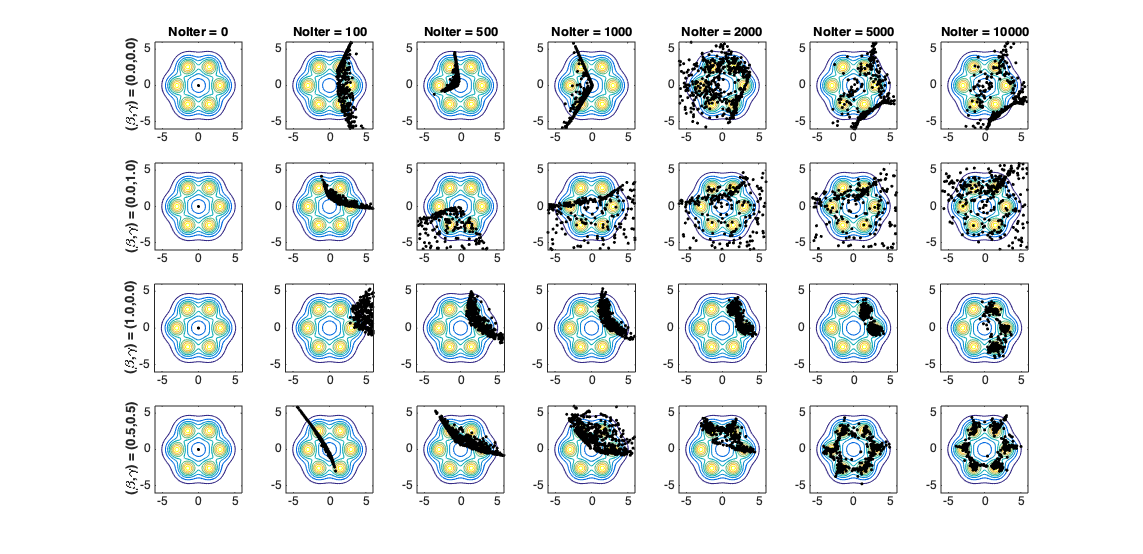}}
\caption{Same as Fig. \ref{tmm6:32units:fig} but with 3 hidden layers and 32 units per layer. Apart from Hellinger distance minimization, all other losses failed to converge in $10K$ iterations.}
\label{tmm6:32units:3layers:fig}
\end{center}
\end{figure}

\clearpage

% \subsubsection{Swiss-roll dataset}
The second dataset is the Swiss-roll (i.e., spiral) dataset. This is a challenging example due to its complex manifold structure. Therefore the number of iterations required for training is increased by one order of magnitude. Figs. \ref{swissroll:32units:fig}-\ref{swissroll:32units:3layers:fig} show the training dynamics for three different architectures. Again, Hellinger distance minimization produced the more accurate and robust results while WGAN was unable to converge in all runs of Swiss-roll dataset.

\begin{figure}[hb]
\begin{center}
\centerline{\includegraphics[width=0.95\columnwidth]{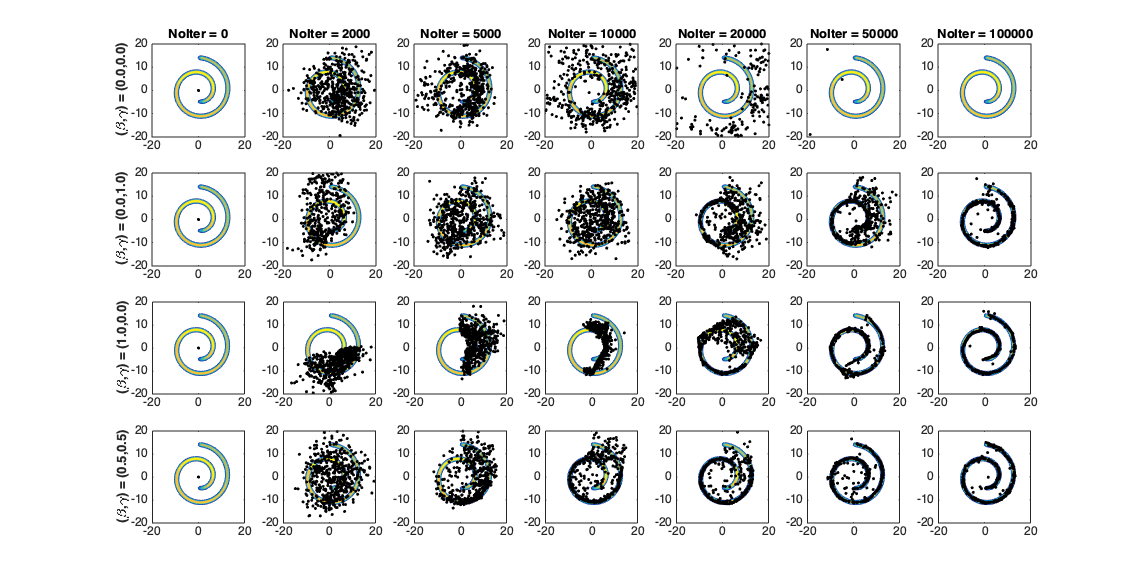}}
\caption{Generated samples using the Wasserstein distance using clipping (1st row), KL divergence (2nd row), reverse KLD (3rd row) and Hellinger distance (last row). 2 hidden layers with 32 units per layer are used. All cumulant losses were able to learn the Swiss-roll distribution.}
\label{swissroll:32units:fig}
\end{center}
\end{figure}

\begin{figure}[h]
\begin{center}
\centerline{\includegraphics[width=0.95\columnwidth]{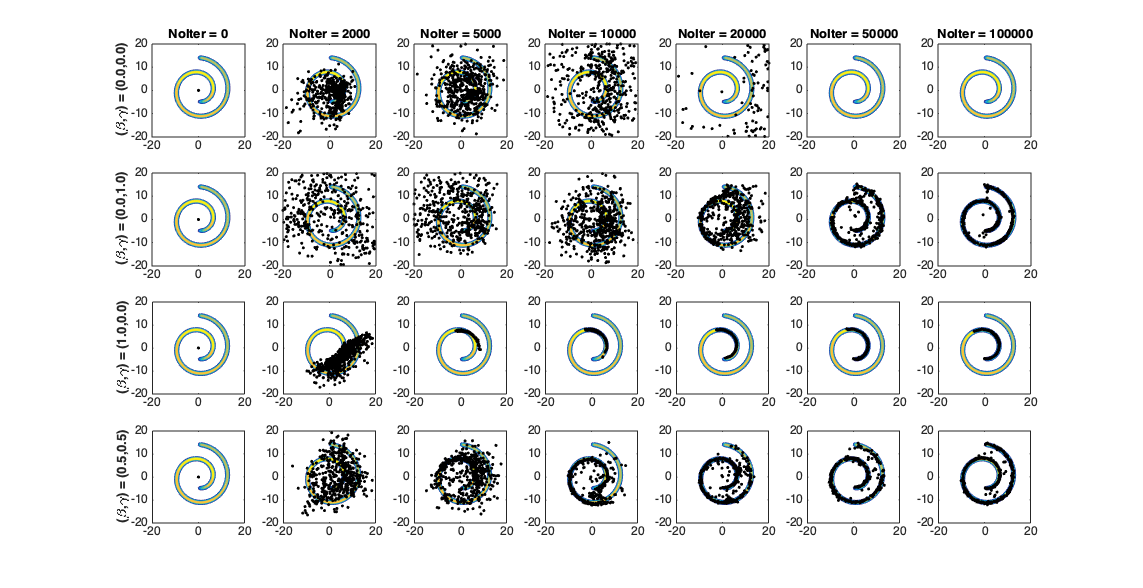}}
\caption{Same as Fig. \ref{swissroll:32units:fig} but with 64 hidden units per layer.}
\label{swissroll:64units:fig}
\end{center}
\end{figure}

\newpage 

\begin{figure}[t]
\begin{center}
\centerline{\includegraphics[width=1\columnwidth]{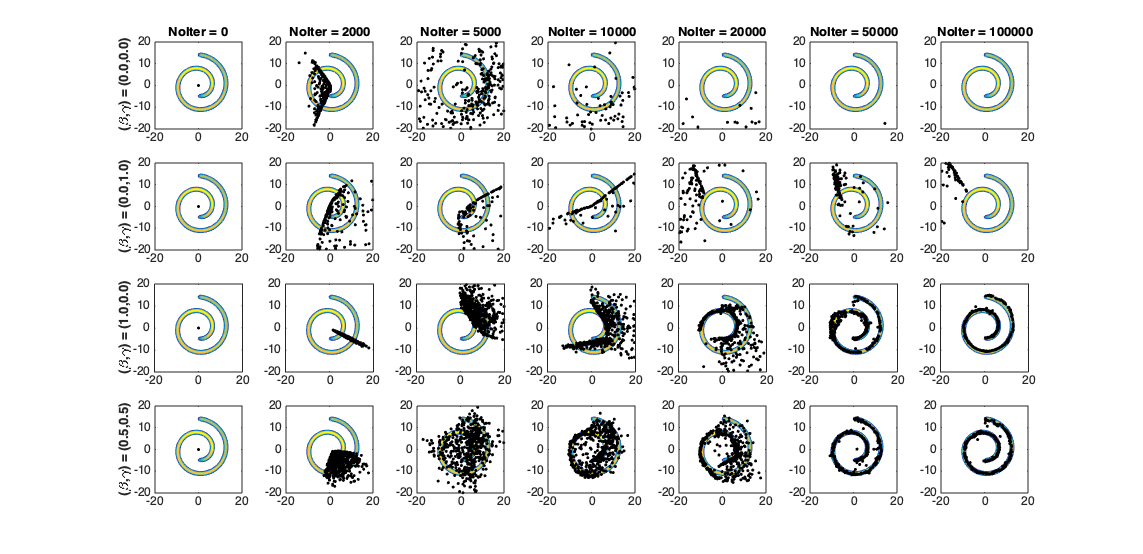}}
\caption{Same as Fig. \ref{swissroll:32units:fig} but with 3 hidden layers and 32 units per layer. As in the previous two figures, Hellinger distance minimization produced the most accurate results.}
\label{swissroll:32units:3layers:fig}
\end{center}
\end{figure}

% \newpage 

\subsection{Exploring the d-rays}
Figs. \ref{KLD:ray:fig}-\ref{Hellinger:ray:fig} present the training dynamics across three different d-rays. We consider KLD (Fig. \ref{KLD:ray:fig}), reverse KLD (Fig. \ref{reverse:KLD:ray:fig}) and Hellinger distance (Fig. \ref{Hellinger:ray:fig}) as loss functions and three values for $\delta$. For each ray, we observe similar qualitatively patterns. Indeed, mode covering is observed when KLD is minimized while mode selection is observed when reverse KLD is minimized. It is also evident that the speed of convergence to the real distribution is different as the last column of both Figs. \ref{KLD:ray:fig} and \ref{reverse:KLD:ray:fig} reveals.

\begin{figure}[ht]
\begin{center}
\centerline{\includegraphics[width=1\columnwidth]{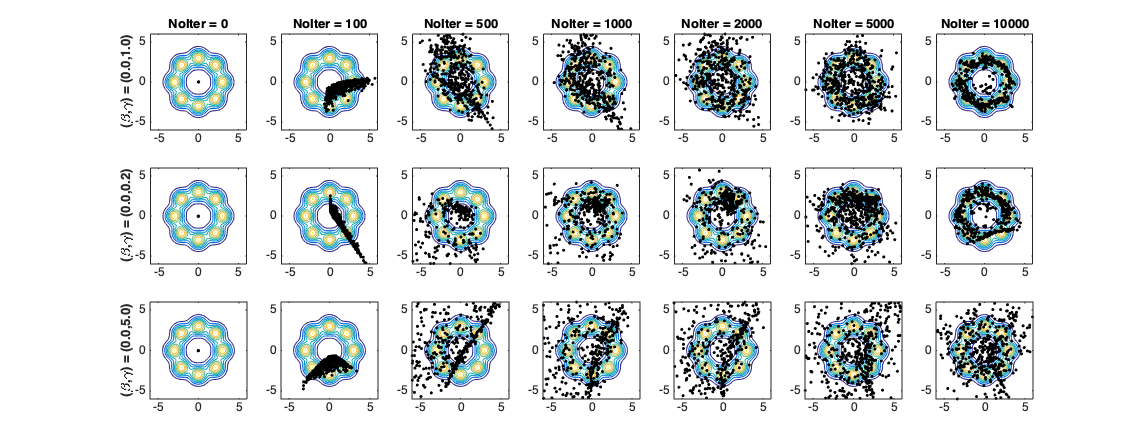}}
\caption{Generated samples at different stages of the training process. Each row correspond to minimizing the KLD (closer to mode covering dynamics) but with different $\delta$ (first row: $\delta=1$, second row: $\delta=0.2$ and third row: $\delta=5$). }
\label{KLD:ray:fig}
\end{center}
\end{figure}

\begin{figure}[t]
\begin{center}
\centerline{\includegraphics[width=1\columnwidth]{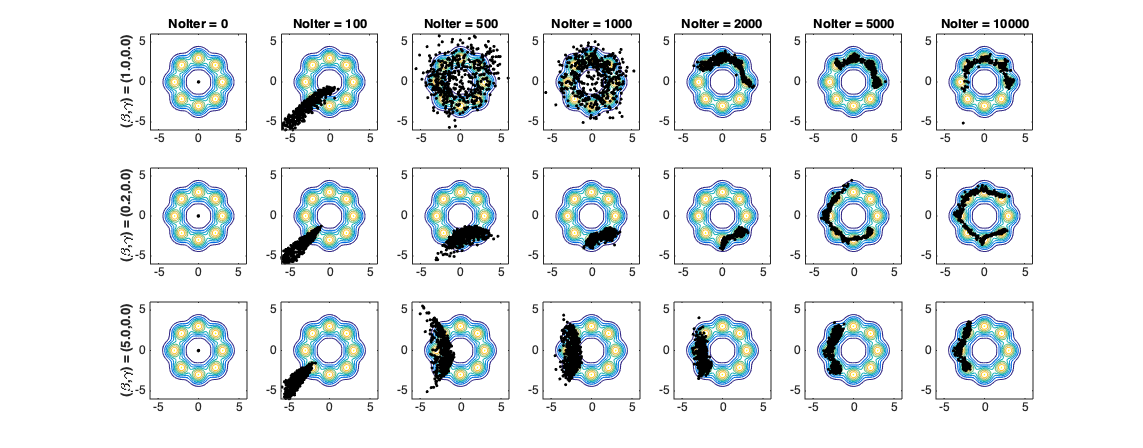}}
\caption{Same as Fig. \ref{KLD:ray:fig} but for the reverse KLD. Model selection of different extend is observed for all $\delta$'s.}
\label{reverse:KLD:ray:fig}
\end{center}
\end{figure}

\begin{figure}[t]
\begin{center}
\centerline{\includegraphics[width=1\columnwidth]{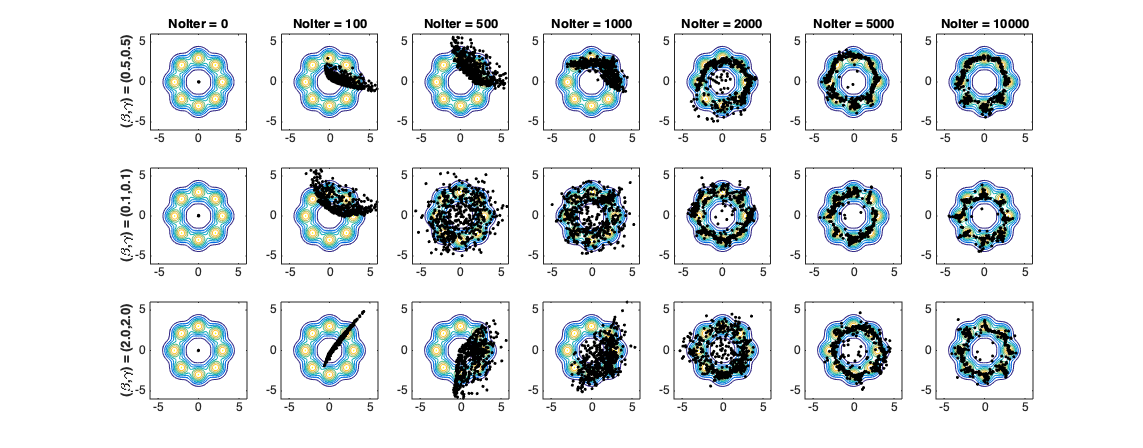}}
\caption{Same as Fig. \ref{KLD:ray:fig} but for the Hellinger distance.}
\label{Hellinger:ray:fig}
\end{center}
\end{figure}

\clearpage
\section{Experimental Details}

Here, we describe the experimental setup and architectural details for all the experiments presented in the paper. Three architectures have been used to compare the performance of four  GAN losses: Wasserstein, Kullback-Leibler
divergence (KLD), reverse KLD and Hellinger distance. The architectures whose successful training we demonstrate are described as follows: (i) convolutional layers for CIFAR-10 data, (ii) residual blocks for CIFAR-10 data (iii) residual blocks for ImageNet data. In the convolutional architecture, batch normalization is applied only for generator but not for discriminator. Whereas, we implemented layer normalization in both generator and discriminator. We used Adam as the optimizer with a learning rate of 0.0001. We trained the model for a total of 200,000 iterations on CIFAR-10 and 500,000 iterations on ImageNet, with a mini-batch of 128 and 64, respectively.

\subsection{CIFAR-10 Convolutional Architecture}

\begin{table}[H]
\centering
\resizebox{0.9\columnwidth}{!}{%
\setlength{\tabcolsep}{15pt}
\renewcommand{\arraystretch}{1.5}
\begin{tabular}{ccccc}
\specialrule{.1em}{.05em}{.05em}
\multicolumn{5}{c}{Generator}                                           \\ \specialrule{.1em}{.05em}{.05em}
Layer            & Kernel & Output shape & Stride & Activation function \\ \specialrule{.1em}{.05em}{.05em}
Input $z$          & -      & 128          & -      & -                   \\
Linear           & -      & $512\times4\times4$      & -      & -                   \\
Transposed convolution 1 & $5\times5$    & $256\times8\times8$      & 1      & ReLU                \\
Transposed convolution 2 & $5\times5$    & $128\times16\times16$    & 1      & ReLU                \\
Transposed convolution 3    & $5\times5$    & $3\times32\times32$      & 1      & tanh                \\ \specialrule{.1em}{.05em}{.05em}
\multicolumn{5}{c}{Discriminator}                                       \\ \specialrule{.1em}{.05em}{.05em}
Convolution      & $5\times5$    & $256\times8\times8$     & 4      &  Leaky ReLU                  \\
Linear           & -      & 1            & -      & -                   \\ \specialrule{.1em}{.05em}{.05em}
\end{tabular}%
}
\end{table}

\subsection{CIFAR-10 Residual Architecture - Weak Discriminator}

\begin{table}[H]
\centering
\resizebox{0.9\columnwidth}{!}{%
\setlength{\tabcolsep}{15pt}
\renewcommand{\arraystretch}{1.2}
\begin{tabular}{ccccc}
\specialrule{.1em}{.05em}{.05em}
\multicolumn{5}{c}{Generator}                                           \\ \specialrule{.1em}{.05em}{.05em} 
Layer            & Kernel & Output shape & Stride & Activation function \\ \specialrule{.1em}{.05em}{.05em}
Input $z$          & -      & 128          & -      & -                   \\
Linear           & -      & $1024\times2\times2$      & -      & -                   \\
Residual block 1 & $3\times3$    & $1024\times4\times4$      & 1      & ReLU                \\
Residual block 2 & $3\times3$    & $512\times8\times8$      & 1      & ReLU                \\
Residual block 3 & $3\times3$    & $256\times16\times16$    & 1      & ReLU                \\
Residual block 4 & $3\times3$    & $128\times32\times32$     & 1      & ReLU                \\
Convolution      & $3\times3$    & $3\times32\times32$      & 1      & tanh                \\ \specialrule{.1em}{.05em}{.05em}
\multicolumn{5}{c}{Discriminator}                                       \\ \specialrule{.1em}{.05em}{.05em}
Convolution      & $3\times3$    & $128\times32\times32$     & 1      & -                   \\
Residual block 1 & $3\times3$   & $256\times16\times16$    & 1      & ReLU                \\
Residual block 2 & $3\times3$    & $256\times8\times8$      & 1      & ReLU                \\
Linear           & -      & 1            & -      & -                   \\ \specialrule{.1em}{.05em}{.05em}
\end{tabular}%
}
\end{table}

\newpage
\subsection{CIFAR-10 Residual Architecture - Strong Discriminator (Version 1)}

\begin{table}[h]
\centering
\resizebox{0.9\columnwidth}{!}{%
\setlength{\tabcolsep}{15pt}
\renewcommand{\arraystretch}{1.2}
\begin{tabular}{ccccc}
\specialrule{.1em}{.05em}{.05em}
\multicolumn{5}{c}{Generator}                                           \\ \specialrule{.1em}{.05em}{.05em} 
Layer            & Kernel & Output shape & Stride & Activation function \\ \specialrule{.1em}{.05em}{.05em}
Input $z$          & -      & 128          & -      & -                   \\
Linear           & -      & $1024\times2\times2$      & -      & -                   \\
Residual block 1 & $3\times3$    & $1024\times4\times4$      & 1      & ReLU                \\
Residual block 2 & $3\times3$    & $512\times8\times8$      & 1      & ReLU                \\
Residual block 3 & $3\times3$    & $256\times16\times16$    & 1      & ReLU                \\
Residual block 4 & $3\times3$    & $128\times32\times32$     & 1      & ReLU                \\
Convolution      & $3\times3$    & $3\times32\times32$      & 1      & tanh                \\ \specialrule{.1em}{.05em}{.05em}
\multicolumn{5}{c}{Discriminator}                                       \\ \specialrule{.1em}{.05em}{.05em}
Convolution      & $3\times3$    & $128\times32\times32$     & 1      & -                   \\
Residual block 1 & $3\times3$   & $256\times16\times16$    & 1      & ReLU                \\
Residual block 2 & $3\times3$   & $256\times16\times16$    & 1      & ReLU                \\
Residual block 3 & $3\times3$   & $256\times16\times16$    & 1      & ReLU                \\
Linear           & -      & 1            & -      & -                   \\ \specialrule{.1em}{.05em}{.05em}
\end{tabular}%
}
\end{table}

\subsection{CIFAR-10 Residual Architecture - Strong Discriminator (Version 2)}

\begin{table}[h]
\centering
\resizebox{0.9\columnwidth}{!}{%
\setlength{\tabcolsep}{15pt}
\renewcommand{\arraystretch}{1.2}
\begin{tabular}{ccccc}
\specialrule{.1em}{.05em}{.05em}
\multicolumn{5}{c}{Generator}                                           \\ \specialrule{.1em}{.05em}{.05em} 
Layer            & Kernel & Output shape & Stride & Activation function \\ \specialrule{.1em}{.05em}{.05em}
Input $z$          & -      & 128          & -      & -                   \\
Linear           & -      & $1024\times2\times2$      & -      & -                   \\
Residual block 1 & $3\times3$    & $1024\times4\times4$      & 1      & ReLU                \\
Residual block 2 & $3\times3$    & $512\times8\times8$      & 1      & ReLU                \\
Residual block 3 & $3\times3$    & $256\times16\times16$    & 1      & ReLU                \\
Residual block 4 & $3\times3$    & $128\times32\times32$     & 1      & ReLU                \\
Convolution      & $3\times3$    & $3\times32\times32$      & 1      & tanh                \\ \specialrule{.1em}{.05em}{.05em}
\multicolumn{5}{c}{Discriminator}                                       \\ \specialrule{.1em}{.05em}{.05em}
Convolution      & $3\times3$    & $128\times32\times32$     & 1      & -                   \\
Residual block 1 & $3\times3$   & $128\times16\times16$    & 1      & ReLU                \\
Residual block 2 & $3\times3$   & $128\times16\times16$    & 1      & ReLU                \\
Residual block 3 & $3\times3$   & $128\times16\times16$    & 1      & ReLU                \\
Linear           & -      & 1            & -      & -                   \\ \specialrule{.1em}{.05em}{.05em}
\end{tabular}%
}
\end{table}

\newpage
\subsection{ImageNet Residual Architecture - Weak Discriminator}

\begin{table*}[h]
\centering
\resizebox{0.9\columnwidth}{!}{%
\setlength{\tabcolsep}{15pt}
\renewcommand{\arraystretch}{1.2}
\begin{tabular}{ccccc}
\specialrule{.1em}{.05em}{.05em}
\multicolumn{5}{c}{Generator}                                           \\ \specialrule{.1em}{.05em}{.05em}
Layer            & Kernel & Output shape & Stride & Activation function \\ \specialrule{.1em}{.05em}{.05em}
Input $z$          & -      & 128          & -      & -                   \\
Linear           & -      & $1024\times4\times4$      & -      & -                   \\
Residual block 1 & $3\times3$    & $1024\times8\times8$      & 1      & ReLU                \\
Residual block 2 & $3\times3$    & $512\times16\times16$      & 1      & ReLU                \\
Residual block 3 & $3\times3$    & $256\times32\times32$   & 1      & ReLU                \\
Residual block 4 & $3\times3$    & $128\times64\times64$     & 1      & ReLU                \\
Convolution      & $3\times3$    & $3\times64\times64$      & 1      & tanh                \\ \specialrule{.1em}{.05em}{.05em}
\multicolumn{5}{c}{Discriminator}                                       \\ \specialrule{.1em}{.05em}{.05em}
Convolution      & $3\times3$    & $128\times64\times64$     & 1      & -                   \\
Residual block 1 & $3\times3$    & $128\times32\times32$    & 1      & ReLU                \\
Residual block 2 & $3\times3$    & $256\times16\times16$      & 1      & ReLU                \\
Linear           & -      & 1            & -      & -                   \\ \specialrule{.1em}{.05em}{.05em}
\end{tabular}%
}
\end{table*}

\subsection{ImageNet Residual Architecture - Strong Discriminator}

\begin{table*}[h]
\centering
\resizebox{0.9\columnwidth}{!}{%
\setlength{\tabcolsep}{15pt}
\renewcommand{\arraystretch}{1.2}
\begin{tabular}{ccccc}
\specialrule{.1em}{.05em}{.05em}
\multicolumn{5}{c}{Generator}                                           \\ \specialrule{.1em}{.05em}{.05em}
Layer            & Kernel & Output shape & Stride & Activation function \\ \specialrule{.1em}{.05em}{.05em}
Input $z$          & -      & 128          & -      & -                   \\
Linear           & -      & $1024\times4\times4$      & -      & -                   \\
Residual block 1 & $3\times3$    & $1024\times8\times8$      & 1      & ReLU                \\
Residual block 2 & $3\times3$    & $512\times16\times16$      & 1      & ReLU                \\
Residual block 3 & $3\times3$    & $256\times32\times32$   & 1      & ReLU                \\
Residual block 4 & $3\times3$    & $128\times64\times64$     & 1      & ReLU                \\
Convolution      & $3\times3$    & $3\times64\times64$      & 1      & tanh                \\ \specialrule{.1em}{.05em}{.05em}
\multicolumn{5}{c}{Discriminator}                                       \\ \specialrule{.1em}{.05em}{.05em}
Convolution      & $3\times3$    & $128\times64\times64$     & 1      & -                   \\
Residual block 1 & $3\times3$    & $256\times32\times32$    & 1      & ReLU                \\
Residual block 2 & $3\times3$    & $512\times16\times16$      & 1      & ReLU                \\
Residual block 3 & $3\times3$    & $1024\times8\times8$    & 1      & ReLU                \\
Residual block 4 & $3\times3$    & $1024\times4\times4$      & 1      & ReLU                \\
Linear           & -      & 1            & -      & -                   \\ \specialrule{.1em}{.05em}{.05em}
\end{tabular}%
}
\end{table*}

\clearpage
\section{Generated Images}
In this section, generated samples from all the trained models are presented. We remark that all models are trained with GP regularization.

\begin{figure*}[h]
\begin{center}
\centerline{\includegraphics[width=18cm,height=16cm]{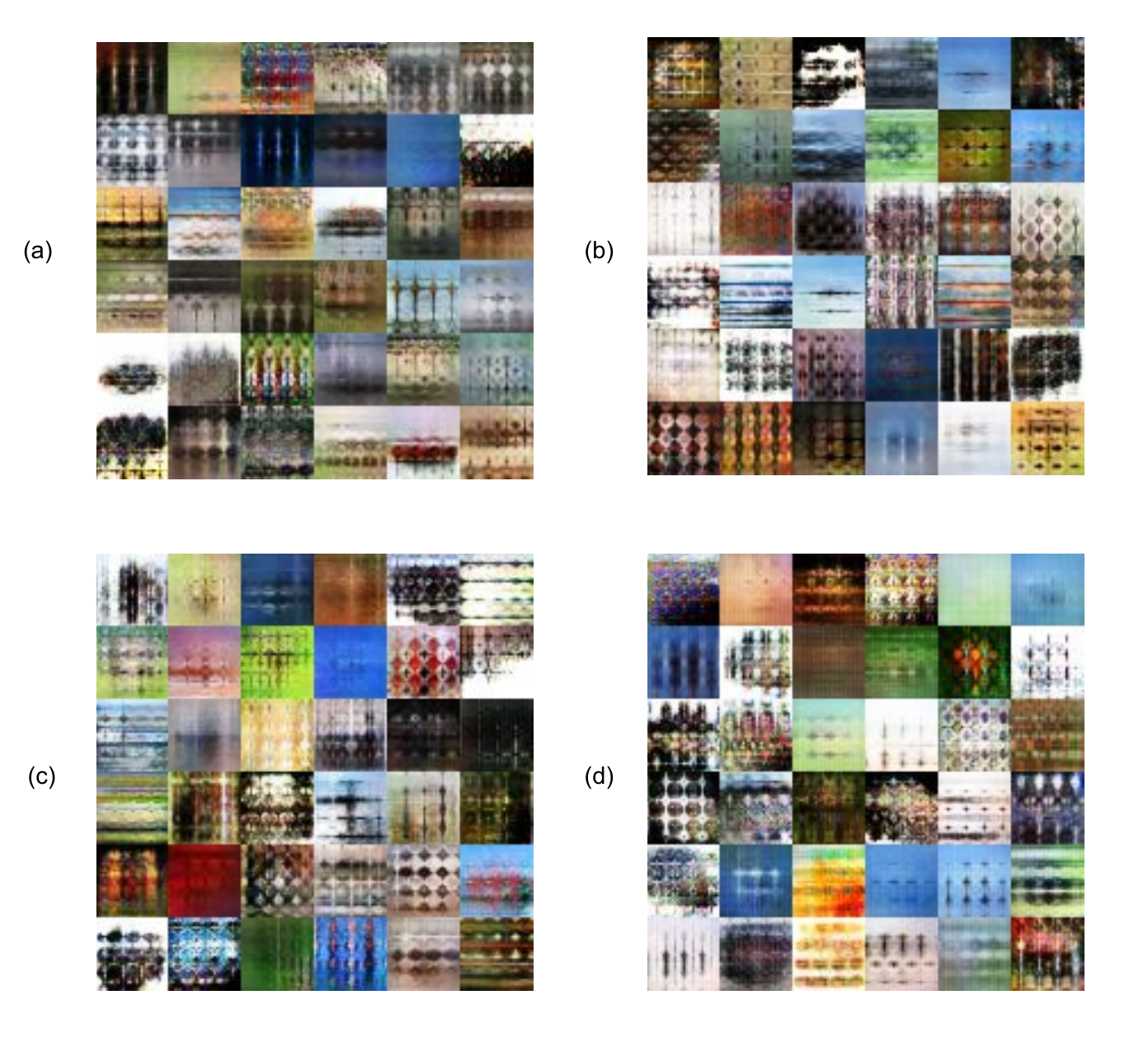}}
\caption{(a) WGAN, (b) KLD, (c) RKLD and (d) Hellinger distance: Samples of CIFAR-10 from generator and discriminator trained with convolutional networks.}
\end{center}
\end{figure*}

\begin{figure*}[t!]
\begin{center}
\centerline{\includegraphics[width=18cm,height=16cm]{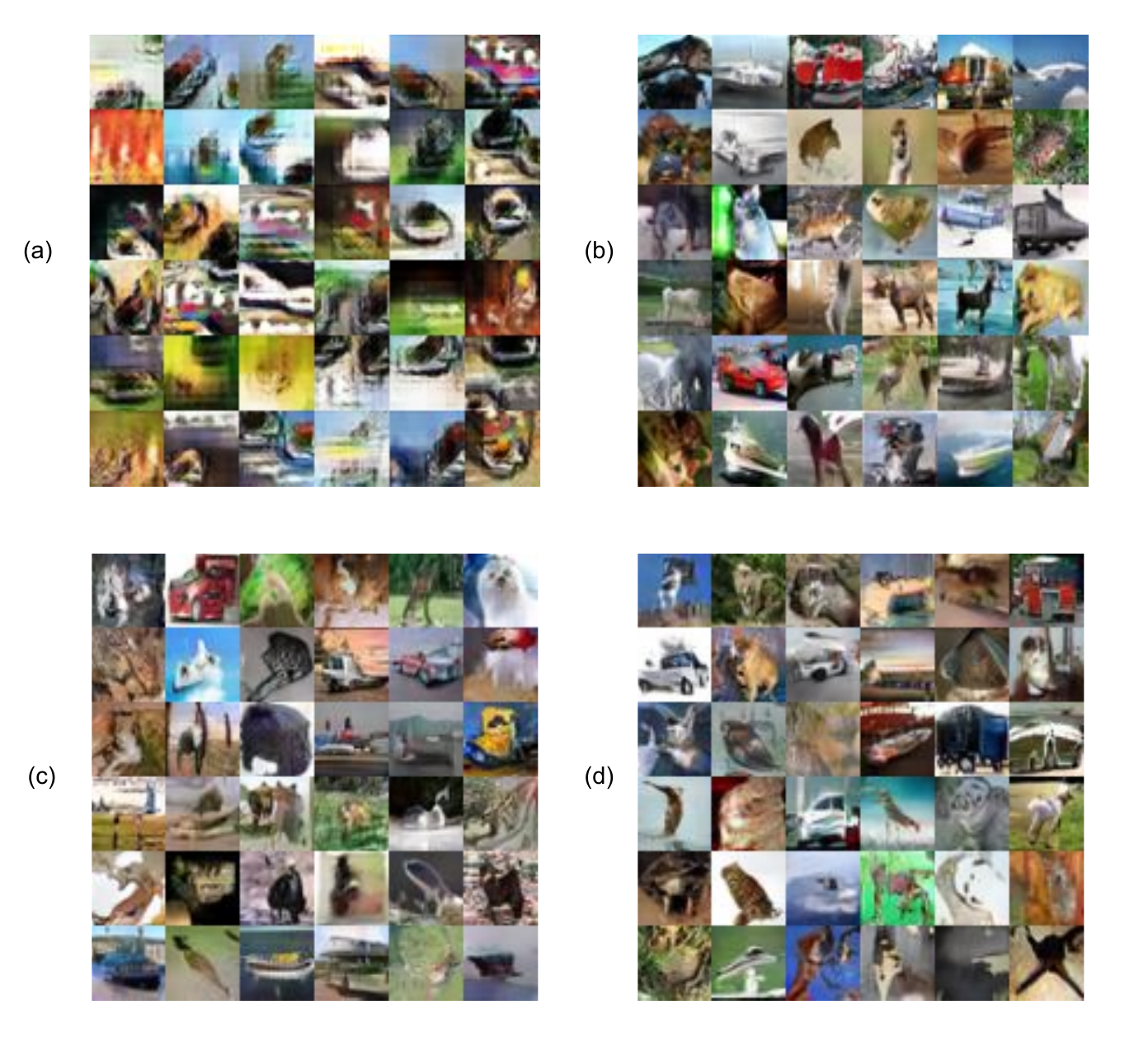}}
\caption{(a) WGAN, (b) KLD, (c) RKLD and (d) Hellinger distance: Samples of CIFAR-10 from generator and discriminator trained with residual network - weak discriminator.}
\end{center}
\end{figure*}

\begin{figure*}[t!]
\begin{center}
\centerline{\includegraphics[width=18cm,height=16cm]{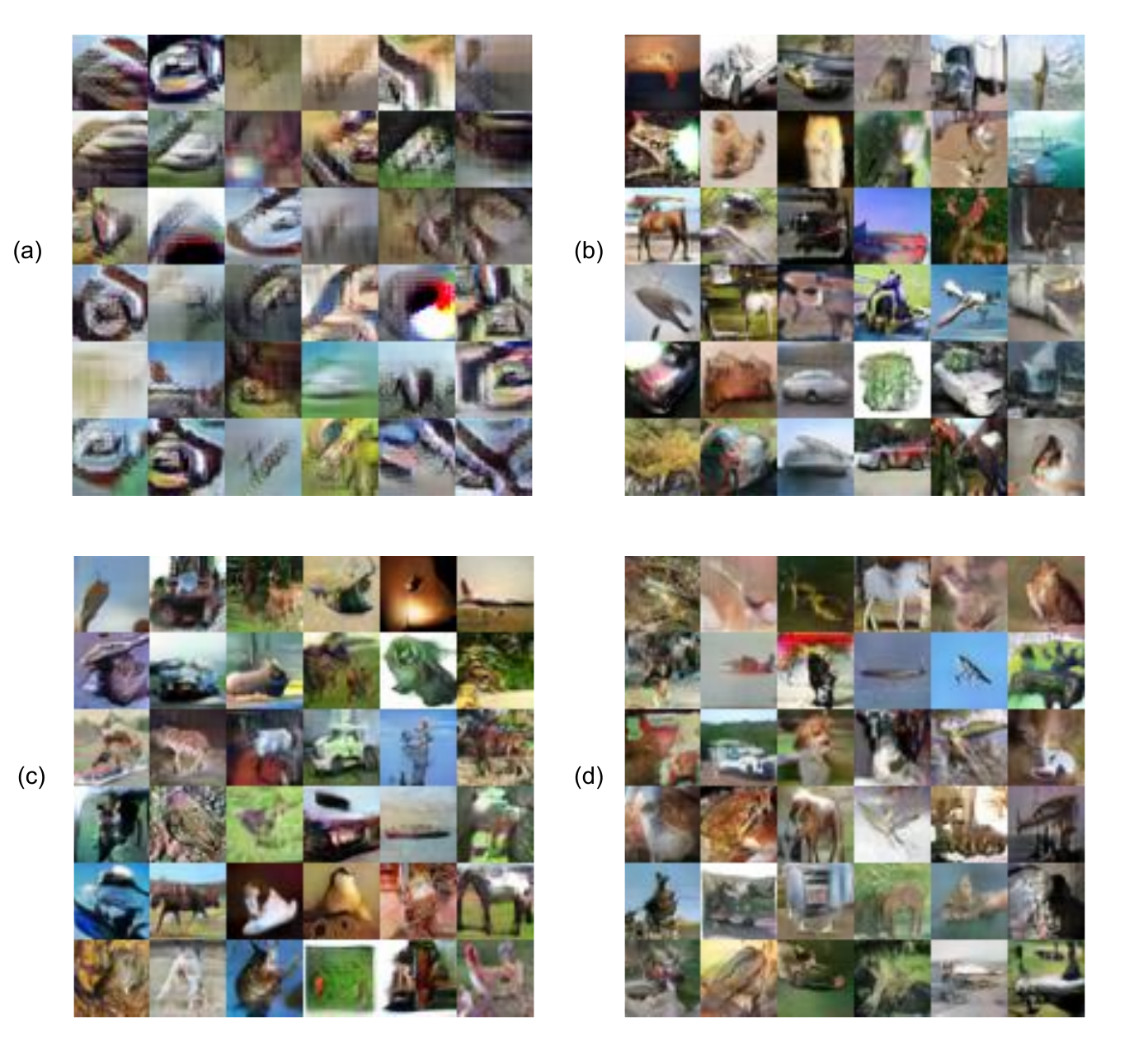}}
\caption{(a) WGAN, (b) KLD, (c) RKLD and (d) Hellinger distance: Samples of CIFAR-10 from generator and discriminator trained with residual network - strong discriminator (version 1).}
\end{center}
\end{figure*}

\begin{figure*}[t!]
\begin{center}
\centerline{\includegraphics[width=18cm,height=16cm]{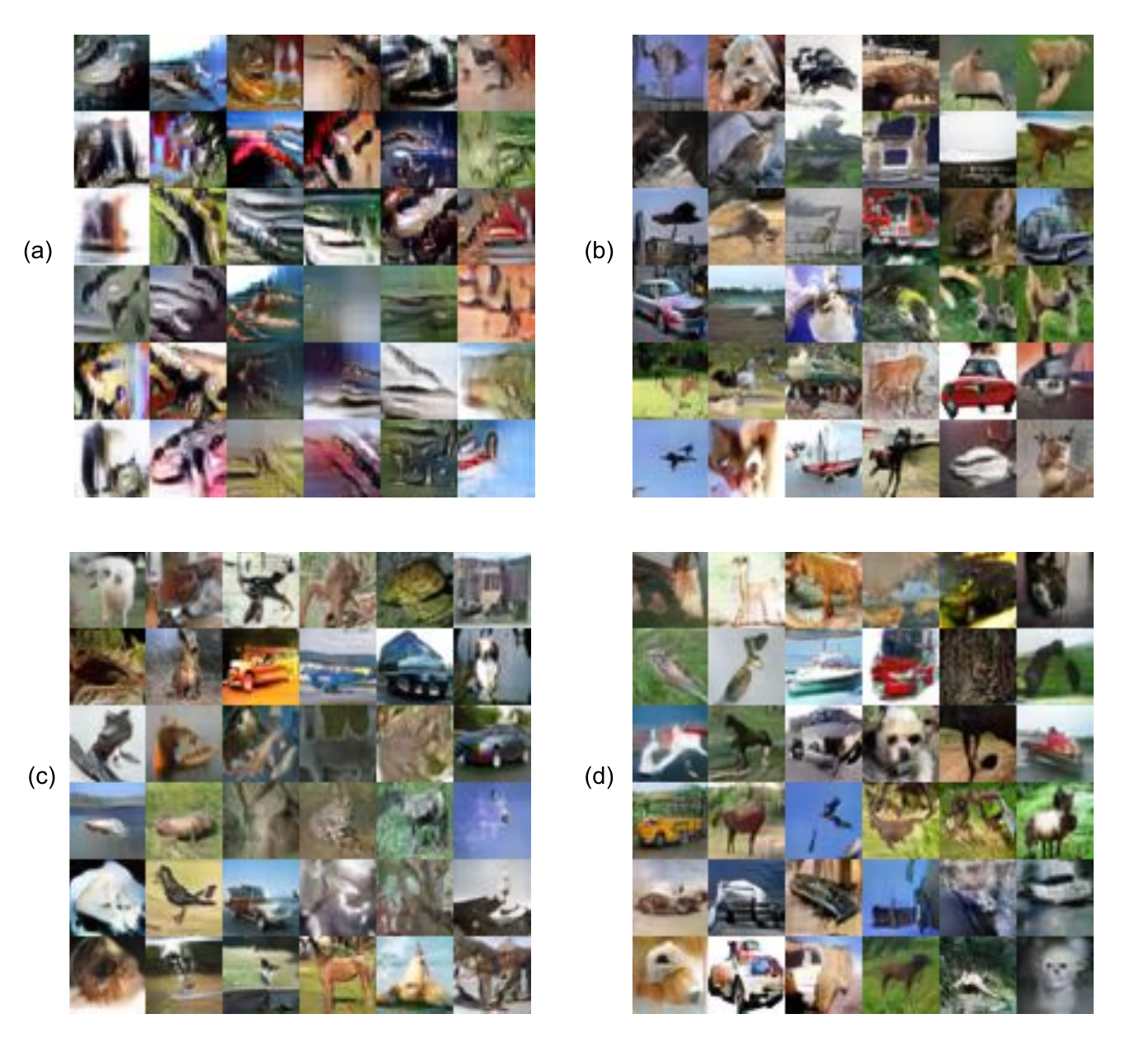}}
\caption{(a) WGAN, (b) KLD, (c) RKLD and (d) Hellinger distance: Samples of CIFAR-10 from generator and discriminator trained with residual network - strong discriminator (version 2).}
\end{center}
\end{figure*}

\begin{figure*}[t!]
\begin{center}
\centerline{\includegraphics[width=18cm,height=16cm]{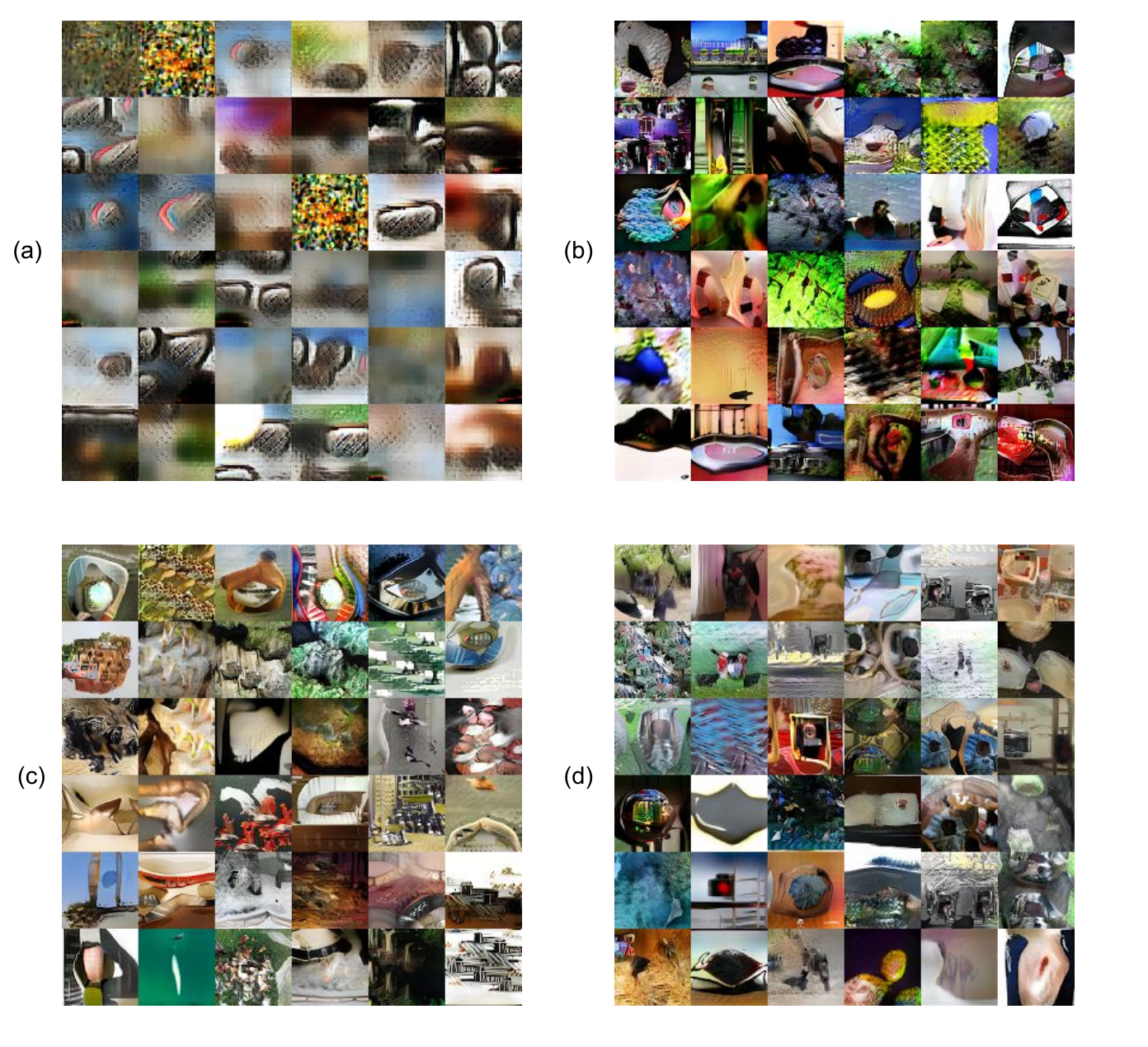}}
\caption{(a) WGAN, (b) KLD, (c) RKLD and (d) Hellinger distance: Samples of ImageNet from generator and discriminator trained with residual networks - weak discriminator.}
\end{center}
\end{figure*}

\begin{figure*}[t!]
\begin{center}
\centerline{\includegraphics[width=18cm,height=16cm]{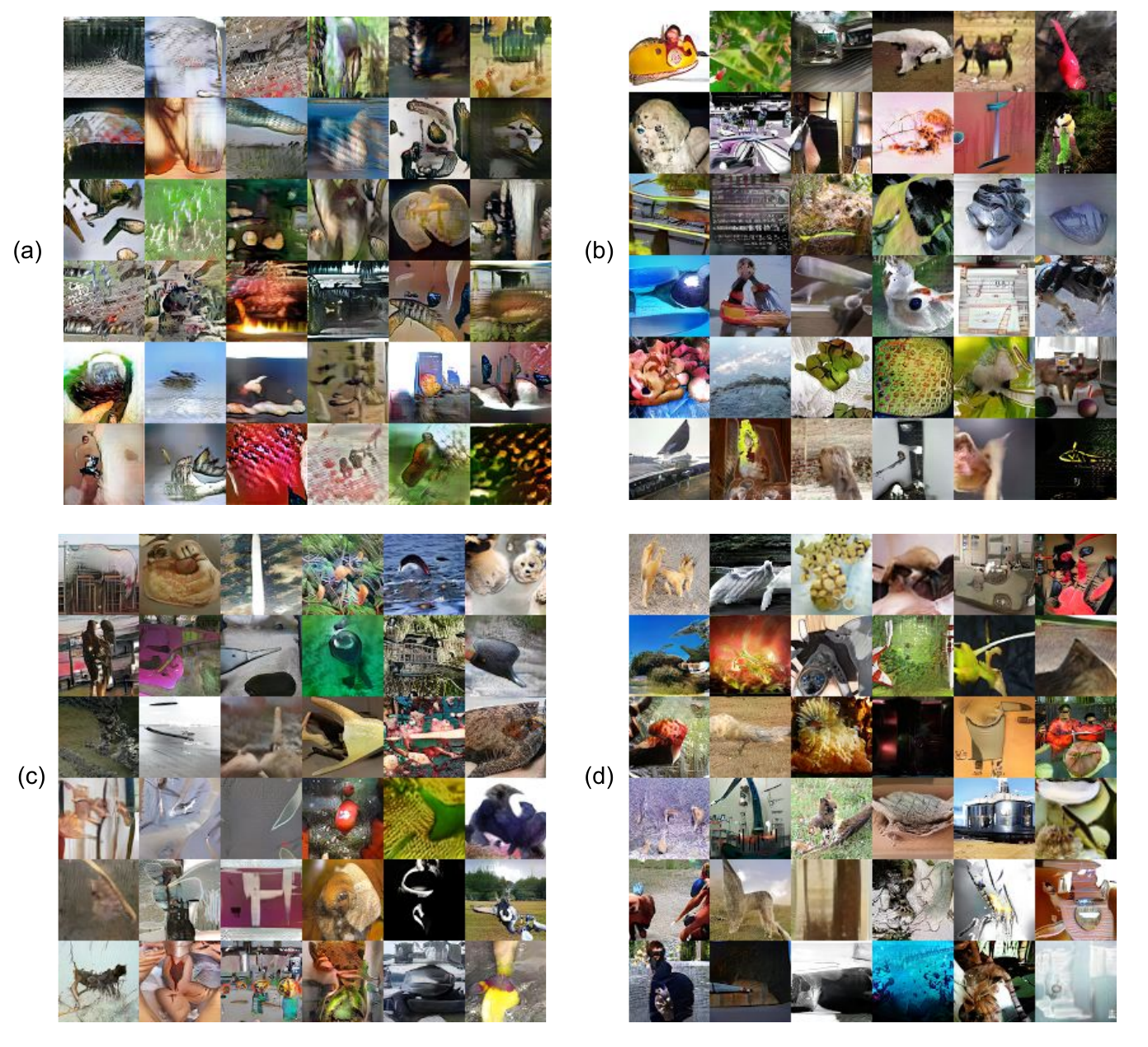}}
\caption{(a) WGAN, (b) KLD, (c) RKLD and (d) Hellinger distance: Samples of ImageNet from generator and discriminator trained with residual networks - strong discriminator.}
\end{center}
\end{figure*}

\end{document}